\documentclass[twoside,11pt]{article}

\usepackage{blindtext}

\usepackage[preprint]{jmlr2e}
\usepackage{booktabs,xcolor,siunitx}
\usepackage{amsmath,amssymb,url,braket,bbm, bm}
\usepackage{stfloats}
\usepackage{lipsum}
\allowdisplaybreaks[4]
\usepackage{hyperref}
\usepackage{mathrsfs}
\usepackage{overpic}
\usepackage{tikz}
\usetikzlibrary{decorations}

\usepackage{adjustbox}
\usepackage{dsfont}
\usepackage{comment}
\usepackage{subfigure}
\usepackage{mathtools}
\usepackage{etoolbox}
\AtBeginEnvironment{dcases}{\selectfont}
\usepackage{enumitem}
\usepackage{graphicx}
\usepackage{algorithm}
\usepackage{algorithmic}

\usepackage{booktabs}
\usepackage{xcolor}
\usepackage{soul}

\DeclareMathOperator{\E}{\mathbb{E}}

\DeclareMathOperator*{\argmax}{arg\,max}

\newcommand{\kl}{\mathrm{kl}}

\newcommand{\cE}{\mathcal{E}}

\newcommand{\cT}{{\mathcal{T}}}

\newcommand{\EE}{\mathbb{E}}

\newcommand{\NN}{\mathbb{N}}
\newcommand{\PP}{\mathbb{P}}

\newcommand{\blambda}{\bm{\lambda}}
\newcommand{\bmu}{\bm{\mu}}

\DeclareMathOperator*{\ind}{\mathds{1}}
\usepackage[pdftex]{pict2e}
\usepackage{lastpage}
\jmlrheading{27}{2026}{1-\pageref{LastPage}}{9/24; Revised 4/26}{5/26}{24-1612}{Junwen Yang, Vincent Y. F. Tan and Tianyuan Jin}

\ShortHeadings{Best Arm Identification with Minimal Regret}{Yang, Tan, and Jin}
\firstpageno{1}

\begin{document}

\title{Best Arm Identification with Minimal Regret}

\author{\name Junwen Yang \email junwen\_yang@u.nus.edu\\
       \addr Institute of Operations Research and Analytics\\
       National University of Singapore\\
       117602, Singapore 
       \AND
       \name Vincent Y. F. Tan \email vtan@nus.edu.sg \\
       \addr Department of Mathematics\\
        Department of Electrical and Computer Engineering\\
        Institute of Operations Research and Analytics\\
       National University of Singapore\\
       119076, Singapore 
        \AND
       \name Tianyuan Jin{\textcolor{blue}{*}}
       \email tianyuan@u.nus.edu \\
       \addr  Department of Mathematics\\
       National University of Singapore\\
       119076, Singapore 
       }
       
\editor{Vianney Perchet}

\maketitle

\begingroup
\renewcommand{\thefootnote}{\fnsymbol{footnote}}
\footnotetext[1]{Corresponding author.}
\endgroup
%TLDR: We consider the problem of best arm identification with minimal regret, and propose an asymtptotically optimal algorithm, termed Double KL-UCB.

\begin{abstract} %
  Motivated by real-world applications that necessitate responsible experimentation, we introduce the problem of \emph{best arm identification (BAI) with minimal regret}. This variant of the multi-armed bandit problem elegantly amalgamates two of its  most ubiquitous objectives: regret minimization and BAI. More precisely, the agent's goal is to identify the best arm with a prescribed confidence level $\delta$, while minimizing the cumulative regret up to the stopping time. Focusing on single-parameter exponential families of distributions, we leverage information-theoretic techniques to establish an instance-dependent lower bound on the expected cumulative regret. Moreover, we present an   impossibility result that underscores the tension between cumulative regret and sample complexity in fixed-confidence BAI. Complementarily, we design and analyze the Double KL-UCB algorithm, which achieves asymptotic optimality as the confidence level tends to zero. Notably, this algorithm employs two distinct confidence bounds to guide arm selection in a randomized manner. Our findings elucidate a fresh perspective on the inherent connections between regret minimization and BAI.
\end{abstract}


\section{Introduction}
\label{section_intro}
The multi-armed bandit model, originally introduced by \citet{thompson1933likelihood}, offers a straightforward yet powerful online learning framework that has undergone extensive investigation in the online decision making literature. 
In the domain of stochastic multi-armed bandits, the agent grapples with the intricate challenge of optimizing arm-pulling decisions to attain specific objectives.
Notably, the \emph{regret minimization} problem aims at maximizing the cumulative rewards (or equivalently, minimizing the cumulative regret) by delicately balancing the trade-off between exploration and exploitation \citep{agrawal1995sample, auer2002finite, bubeck2012regret}. On the other hand, the \emph{best arm identification} (BAI) problem focuses on achieving efficient  exploration of the optimal arm \citep{even2006action, audibert2010best, karnin2013almost, garivier2016optimal}. 

Nevertheless, despite being arguably two of the most prominent problems in the studies of multi-armed bandits, the performance metrics for BAI and regret minimization are distinct and merit rather different considerations in the design of algorithms and their analyses. In the context of fixed-confidence BAI, the agent seeks to identify the best arm with a given (high) confidence level. Here, the conventional emphasis lies solely on the {\em sample complexity}. In other words, the agent strives to utilize as few samples as possible to achieve the goal of identifying the best arm, without consideration of the cumulative regret. While widely recognized to be important, this problem formulation may not be applicable to all practical settings, as it assumes that the cost of pulling each arm is the same. In real-world scenarios, cognizant of the fact that the cost of pulling an arm is characterized by the suboptimality gap of that arm (which generally differs across arms), the agent might desire to prioritize the minimization of the {\em cumulative regret} during the interaction with the environment. This approach reflects a commitment to a more responsible experimental process.

%[homogeneity in the costs associated with pulling different arms]??. In real-world scenarios, the agent often prioritizes cumulative regret during the interaction with the environment, reflecting a commitment to a more responsible experimental process.  \textcolor{red}{[this last sentence does not gel with the setnence endding in ??]}

Consider, for example, the field of clinical trials, where researchers aim to discover the most effective medical treatment. However, focusing solely on the quantity of trials conducted, which corresponds to sample complexity, overlooks a crucial aspect---the different impact of various treatments on patients. Consequently, the overall cumulative regret of the study emerges as a more ethically grounded and patient-centric performance metric.

In light of this disparity between existing theoretical formulations and practical considerations, we investigate the problem of \emph{best arm identification with minimal regret}, where the expected cumulative regret up to the stopping time serves as the performance metric of BAI. Through a comprehensive theoretical examination of this crucial problem, our work not only addresses a significant gap in existing bandit studies but also offers a fresh, novel, and somewhat surprising perspective on the underlying connection between regret minimization and BAI.

\paragraph{Main contributions.} Our main results and contributions are summarized as follows:
\begin{enumerate}[label = (\roman*)] 

\item In Section~\ref{section_lower}, we derive fundamental information-theoretic limits for the problem of BAI with minimal regret, as rigorously formulated in Section~\ref{section_setup}.  Specifically, utilizing the change-of-measure technique, we establish an instance-dependent lower bound on the expected cumulative regret for all feasible online BAI algorithms. Furthermore, we derive an impossibility result, as outlined in Theorem~\ref{theorem_impossible}. This result underscores the inevitability that achieving minimal cumulative regret in BAI necessitates a higher-order sample complexity.

\item In Sections~\ref{section_algo} and \ref{section_theory}, we propose and analyze the Double KL-UCB (or \textsc{DKL-UCB}) algorithm. Specifically, at each time step, \textsc{DKL-UCB} selects two candidate arms based on two meticulously designed Kullback--Leibler upper confidence bounds, and then employs an appropriate amount of randomization to pull one of the candidate arms.
We establish that \textsc{DKL-UCB} achieves asymptotic optimality, wherein its expected cumulative regret attains the lower bound as the error probability $\delta$ tends to zero. Additionally, we  demonstrate that its expected sample complexity is nearly optimal, up to a small multiplicative factor that scales as the square of   $\log\log(1/\delta)$.

\item In Section~\ref{section_compare_other},  we conduct a comparative analysis between our examined problem and two other classical problems: regret minimization and BAI with minimal samples, with a focus on their respective asymptotic performances.  
We find that the hardness parameter for regret minimization aligns seamlessly with that of our specific problem, warranting further investigation. Besides, we illustrate the significant distinctions that arise in the context of BAI when employing different performance metrics.

\end{enumerate}

%\subsection{Related Work}
\paragraph{Related work.} Both the problems of cumulative regret minimization and best arm identification have attracted considerable attention in the literature. The asymptotic lower bound for regret minimization was firstly established by \citet{lai1985asymptotically}, and further generalized by \citet{burnetas1996optimal}. Subsequently, a diverse range of policies, such as KL-UCB \citep{garivier2011kl, cappe2013kullback} and Thompson Sampling \citep{agrawal2012analysis, korda2013thompson}, have been meticulously developed to achieve asymptotic optimality across various scenarios. For an extensive survey of regret minimization algorithms, we refer to \citet{lattimore2020bandit}.

%Regarding BAI, there exist two complementary settings: the fixed-confidence and fixed-budget settings, which differ generally and do not yield interchangeable results. 
In this work, we focus on the fixed-confidence BAI problem \citep{even2002pac}. A considerable body of research has been dedicated to establishing upper and lower bounds on the sample complexity of fixed-confidence BAI \citep{kalyanakrishnan2012pac,gabillon2012best,jamieson2014lil,kaufmann2016complexity}, and the gap was ultimately closed in \citet{garivier2016optimal}. Specifically, this insightful work provides a complete characterization of the expected sample complexity of BAI {as the confidence level goes to zero}, and achieves the minimal sample complexity through a tracking strategy named Track-and-Stop.

While this paper focuses on the fixed-confidence setting, it is worth noting that BAI can also be explored under different paradigms: (1) the fixed-budget setting \citep{audibert2010best, karnin2013almost, carpentier2016tight, degenne2023existence}, which consists in minimizing error probability with a predetermined exploration budget; and (2)  minimization of the simple regret \citep{bubeck2009pure, lattimore2016causal, zhao2023revisiting}, which consists in minimizing the expected regret of the selected arm after exploration. These frameworks are generally non-interchangeable and  lead to different theoretical insights and performance guarantees.

Consistent with our intrinsic motivation, several works interpolate the objectives of regret minimization and BAI.
Among these, the work most closely related to ours is  \citet{degenne2019bridging}. In particular, a dual-objective algorithm $\text{UCB}_{\alpha}$ was introduced and analyzed, where the parameter $\alpha$ serves to externally balance between regret and sample complexity. However, we remark that  \citet{degenne2019bridging} only established achievability results, with no lower bounds provided. In contrast, our work focuses on the minimization of expected cumulative regret for BAI, and presents an algorithm that asymptotically attains the fundamental lower bound; hence, tightness and optimality of the results are ensured. %We will compare our algorithm with $\text{UCB}_{\alpha}$ both theoretically and empirically in the main text. 
Subsequently, akin to \citet{degenne2019bridging}, \citet{zhong2023achieving} quantified the trade-off between regret minimization and BAI in the fixed-budget setting. Furthermore, \citet{zhang2023fast} studied whether the asymptotically optimal algorithms for the problem of regret minimization can commit to the best arm quickly. Their perspective is motivated by the practical preference for quick commitment over continuous exploration, and their results are not directly comparable to ours. More recently, \citet{qin2024optimizing} considered a model where the agent can choose to stop experimenting adaptively before reaching the total time horizon. In this model, the objective is to  optimize concurrently both within-experiment and post-experiment cost functions. \citet{kanarios2024cost} studied cost-aware BAI, where each arm is associated with a cost distribution with support in $[\ell,1]$ where $\ell>0$, and the agent aims to minimize the cumulative cost up to the stopping time. A crucial difference of our framework is that the cost associated with the optimal arm is exactly zero in the context of cumulative regret, leading to fundamentally different theoretical results.

\section{Problem Setup and Preliminaries}
\label{section_setup}

\paragraph{Multi-armed bandits. }  We consider a conventional stochastic multi-armed bandit model with finitely many arms. Specifically, the arm set is denoted as $[K]:=\{1,2, \ldots, K\}$ and each arm $i$ is associated with a reward distribution $\nu_i$ with mean $\mu_i \in \mathbb R$. We adopt the assumption that the reward distributions $\nu_i$ for all $i\in [K]$ belong to a known canonical single-parameter exponential family. A formal introduction to this family will follow shortly, with a crucial takeaway being that the bandit instance can be fully characterized by the means of its arms, expressed as the vector $\bmu = (\mu_1, \mu_2, \ldots, \mu_K) \in \mathbb R ^K$. 
For convenience, we assume that arm $1$ is the unique\footnote{The uniqueness assumption is commonly applied in fixed-confidence BAI, as it is impossible to distinguish between two arms with identical means.} optimal arm, i.e., $1 = \argmax _{i\in [K]} \mu_i $. For each arm $i\in[K]$, we define $\Delta_i := \mu_1 - \mu_i$ as its \emph{suboptimality gap}. Furthermore, let $\mathcal{M}$ denote the collection of all bandit instances defined above.

At each time step $t\in\mathbb N$, the agent chooses an arm $A_t$ from the given arm set $[K]$. Subsequently, the agent observes a random reward $X_t$, which is drawn from its associated distribution $\nu_{A_t}$ and independent of the rewards obtained from previous time steps. Notably, subsequent arm selections, in general, depend on both prior arm choices and preceding rewards.

\paragraph{Best arm identification with minimal regret. } In the fixed-confidence setting where a confidence level $\delta \in(0,1)$ is given, the agent aims to identify the best arm with a probability of at least $1-\delta$ and minimal expected cumulative regret. This is achieved through sequential and adaptive arm pulling.

More formally, the agent employs an \emph{online algorithm} $\pi$ to decide the arm $A_t$ to pull at each time step $t$, to select a time $\tau_\delta$ to stop pulling arms, and to ultimately recommend $i_{\mathrm{out}}$ as the identified best arm to output. Let $\mathcal F_t :=\sigma(A_1,X_1,\ldots,A_t,X_t)$ denote the $\sigma$-field generated by the interaction history up to and including time $t$. Thus, the online algorithm $\pi$ consists of three integral components:
\begin{itemize}[itemsep=0pt,topsep=0pt,parsep=1pt]
    \item The \emph{sampling rule} selects arm $A_t$, which is $\mathcal F_{t-1}$-measurable;
    \item The \emph{stopping rule} determines a stopping time $\tau_\delta$, which is adapted to the filtration~$(\mathcal F_t)_{t=1}^{\infty}$;
    \item The \emph{recommendation rule} produces a candidate best arm $i_{\mathrm{out}}$, which is $\mathcal F_{\tau_\delta}$-measurable.
\end{itemize}

\begin{definition}
For a prescribed confidence level $\delta\in (0,1)$, an online best arm identification algorithm $\pi$ is said to be {\em $\delta$-PAC (probably approximately correct)} if, for all bandit instances $\bmu \in \mathcal{M}$, it terminates within a finite time almost surely and the probability of error is no more than $\delta$, i.e., $ \mathbb P _{\bmu} ( i_{\mathrm{out}} \neq 1 ) \le \delta$.
\end{definition}

Let $R(t) := \sum_{s=1}^{t} (\mu_1 - X_s)$ represent the cumulative regret up to any time step $t\in\mathbb N$. Then our overarching goal is to design and analyze a $\delta$-PAC online BAI algorithm  while minimizing its expected cumulative regret $\mathbb E_{\bmu}[R(\tau_\delta)]$. Essentially, we seek to address the following problem:
\begin{equation}
\label{equation_setup}
\begin{aligned}
    \min \quad  &\mathbb E_{\bmu}[R(\tau_\delta)] \\
    \text{s.t.}  \quad &\mathbb P _{\bmu}(\tau_\delta<\infty) = 1 \text{ and } \mathbb P _{\bmu} ( i_{\mathrm{out}} \neq 1 ) \le \delta
\end{aligned}
\end{equation}
where $R(\tau_\delta) = \sum_{t=1}^{\tau_\delta} (\mu_1 - X_t)$ and the minimization is over all online algorithms as defined above.

\begin{remark}
The problem under investigation is closely related to two classical challenges in the multi-armed bandit literature: cumulative regret minimization and (fixed-confidence) BAI with minimal samples. Although these problems are typically analyzed within the same bandit framework using exponential family distributions, their fundamentally different goals lead to distinct methodological approaches. A detailed comparative discussion can be found in Section~\ref{section_compare_other}.
\end{remark}

\paragraph{Exponential families. } The canonical single-parameter exponential families of distributions, which encompass a wide range of common distributions such as Gaussian (with fixed variance), Bernoulli, exponential, and Gamma (with fixed shape parameter), have been widely embraced in the bandit literature \citep{agrawal1995sample, cappe2013kullback, korda2013thompson, garivier2016optimal}. 
The distribution $\nu_{\theta}$ of one such family is absolutely continuous with respect to some reference measure (e.g., the Lebesgue measure) $\rho$ on $\mathbb{R}$, with the density function
\begin{equation*}
\frac{\mathrm{d} \nu_\theta}{\mathrm{d} \rho}(x)=\exp (\theta x-b(\theta)),
\end{equation*}
where {the log-partition function $b(\theta)=\log \left(\int_{\mathbb{R}} \exp(\theta x)\, \mathrm{d} \rho(x)\right)$ is infinitely differentiable}, and the parameter $\theta \in \varTheta:=\{\theta \in \mathbb{R}: b(\theta)<\infty\}$. Then it can be verified that the mean and variance of $\nu_{\theta}$ are equal to $b'(\theta)$ and $b''(\theta) > 0$, respectively. Therefore, the mapping between the parameter $\theta$ and the mean value of  $\nu_{\theta}$ is one-to-one. In other words, a distribution $\nu_{\theta}$ of the family can be uniquely characterized by its mean. Let $I \subseteq \mathbb{R}$
represent the collection of means of $\nu_{\theta}$ for all $\theta \in \varTheta$. Furthermore, we assume that the variances in the exponential family are bounded by a constant $V > 0$.\footnote{{This assumption can be relaxed by assuming that $V$ is an upper bound on the variance of reward distributions with means in the interval $[\min_{i \in [K]} \mu_i, \max_{i \in [K]} \mu_i]$. This is achieved by replacing the maximal inequality in Lemma \ref{lem:maximal-inequality} with the adaptive version provided in Lemma B.5 of \cite{jin2024optimal}. For clarity of exposition, we assume bounded variances within the exponential family.}} In this work, the concept of Kullback--Leibler (KL) divergence is employed extensively. For two distributions $\nu_{\theta}$ and $\nu_{\theta'}$ with means $\mu$ and $\mu'$, their KL divergence is denoted as $\mathrm{kl}\left(\mu, \mu' \right)$ and is explicitly given by:
\begin{equation*}
\mathrm{kl}\left(\mu, \mu' \right)=b (\theta' )-b(\theta)-b'(\theta) (\theta'-\theta ).
\end{equation*}
Several straightforward yet crucial properties include that $\mathrm{kl}(\mu, \mu') = 0$ if and only if $\mu = \mu'$, and with $\mu$ fixed, $\mathrm{kl}(\mu, \mu')$ is monotonically increasing  in $\mu'$ for all   $\mu' \geq \mu$.
Additional useful properties of the KL divergence between exponential family distributions are detailed in Appendix~\ref{appendix_auxiliary}. For a comprehensive introduction to exponential families, we refer to \citet{lehmann2006theory}.

\paragraph{Other notations.} Let $\mathcal P _{K} := \{x\in [0, 1]^K: \|x\|_1=1 \}$ denote the probability simplex in $\mathbb R ^ K$. For any $\mu,\mu'\in(0,1)$, the KL divergence between two Bernoulli distributions with means $\mu$ and $\mu'$ is denoted as $\mathrm{kl}_{\mathcal B}(\mu, \mu') :=\mu \log(\mu/\mu') + (1-\mu) \log((1-\mu)/(1-\mu'))$. 
For each arm $i \in [K]$, let $N_i(t) := \sum_{s=1}^t \mathbbm 1 \{ A_{s} = i \}$ and $\hat \mu_i (t):= \sum_{s=1}^t X_{s}\mathbbm 1 \{ A_{s} =i \} / N_i(t)$ denote its total number of  pulls and empirical estimate of the mean up to time $t$, respectively. 
Additionally, throughout this paper, we adopt standard asymptotic notations, including little o, big O, and little omega, always with respect to the confidence level $\delta$ (tending to zero). Specifically, $f(\delta) = \omega(g(\delta))$ means that $f$ grows strictly faster than $g$ as $\delta \to 0^+$, i.e., $\liminf_{\delta \to 0^+} |f(\delta) / g(\delta)| = \infty$.

%{In particular, we set $\hat{\mu}_i(t) = +\infty $ if $N_i(t)= 0$.}  
%Additionally, we define $\hat{\mu}_{is}$ as the empirical mean of arm $i$ based on its first $s$ pulls.

% $\omega()$ notation, $\Theta()$ notation, $o()$ notation

% $\simeq$ notation

\section{Lower Bound}
\label{section_lower}
In this section, we explore the fundamental limits for the problem of best arm identification (BAI) with minimal regret. Specifically, we establish an instance-dependent information-theoretic lower bound on the expected cumulative regret $\mathbb E_{\bmu}[R(\tau_\delta)]$ for all $\delta$-PAC BAI algorithms. Furthermore, we derive an impossibility result regarding the expected sample complexity $\mathbb{E}_{\bmu}[\tau_{\delta}]$ of any \emph{asymptotically optimal} BAI algorithm. These findings provide crucial insights for the
design of our algorithm, which will be introduced in Section~\ref{section_algo}.

We first state the lower bound on the expected cumulative regret $\mathbb E_{\bmu}[R(\tau_\delta)]$ in the following. 

\begin{theorem}[Information-theoretic lower bound]
\label{theorem_lowerbound}
For a fixed confidence level $\delta\in (0,1)$ and instance $\bmu \in \mathcal{M}$, any $\delta$-PAC BAI algorithm satisfies that
\begin{equation*}
\E_{\bmu}[ R(\tau_\delta)] \geq \mathrm I^{*}(\bmu) \mathrm{kl}_{\mathcal B}(\delta, 1-\delta)
\end{equation*}
where
\begin{equation}
\label{equation_Istar}
\mathrm I^{*}(\bmu) := \sum_{i=2}^{K} \frac {\Delta_i} {\mathrm{kl} \left(\mu_{i}, \mu_1 \right)}.
\end{equation}
Furthermore, 
\begin{equation}
\label{equation_asymptotic_lowerbound}
\liminf_{\delta\to 0} \frac{\E_{\bmu}[ R(\tau_\delta)]} {\log (1/\delta)} \ge \mathrm I^{*}(\bmu).
\end{equation}
\end{theorem}

The proof of Theorem~\ref{theorem_lowerbound} is deferred to Appendix~\ref{appendix_proof_theorem_lowerbound}, which employs the ubiquitous change-of-measure argument, dating back to \citet{chernoff1959sequential}, along with some simplifications using optimization techniques.

We refer to $\mathrm I^{*}(\bmu)$ as the {\em hardness parameter} for the problem of BAI with minimal regret. This quantity is a function of the suboptimality gaps and the KL divergences between arm $i$ and the best arm for all suboptimal arms $i>1$.
In the asymptotic scenario where the confidence level $\delta$ tends to zero, the expected cumulative regret $\mathbb E_{\bmu}[R(\tau_\delta)]$ is lower bounded by $\mathrm I^{*}(\bmu) {\log (1/\delta)}$. It is worth noting that this result is tight in view of the theoretical performance of our algorithm {\sc DKL-UCB} presented in Section~\ref{section_algo}. To formalize this, we introduce the concept of asymptotic optimality as follows.

\begin{definition}[Asymptotic optimality] 
\label{definition_optimal}
For the problem of BAI with minimal regret, a $\delta$-PAC algorithm is said to be asymptotically optimal if, for all bandit instances $\bmu \in \mathcal{M}$, its expected cumulative regret matches the lower bound asymptotically, i.e.,  
\begin{equation*}
\limsup_{\delta\to 0} \frac{\E_{\bmu}[ R(\tau_\delta)]} {\log (1/\delta)} \le \mathrm I^{*}(\bmu).
\end{equation*}
\end{definition}

Apparently, the cumulative regret of any asymptotically optimal BAI algorithm is on the order of $\Theta({\log (1/\delta)})$. One may wonder whether there exist other common properties shared by all asymptotically optimal algorithms. In particular, is it possible for the sample complexity (stopping time) of an asymptotically optimal algorithm to attain the same order of $\Theta({\log (1/\delta)})$? We answer this question in the negative through the subsequent impossibility result. % See Appendix~\ref{appendix_proof_theorem_impossible} for the proof of Theorem~\ref{theorem_impossible}, which is built on the intermediate analysis in the proof of Theorem~\ref{theorem_lowerbound}. 

\begin{theorem}[Impossibility result]
\label{theorem_impossible}
For the problem of BAI with minimal regret, any asymptotically optimal $\delta$-PAC algorithm satisfies  
\begin{equation*}
\mathbb{E}_{\bmu}[\tau_{\delta}] = \omega({\log (1/\delta)})
\end{equation*}
for all bandit instances $\bmu \in \mathcal{M}$.
\end{theorem}

% Theorem~\ref{theorem_impossible} demonstrates that the pursuit of identifying the optimal arm while minimizing the expected cumulative regret necessitates a sample complexity on the order of $\omega({\log (1/\delta)})$. This outcome, in fact, aligns with our intuitive understanding of arm pulling. Since pulling the best arm results in zero regret, it is appealing for the agent to select the optimal arm as frequently as possible to acquire a precise estimate of its mean reward. Consequently, this results in a sample complexity that exceeds the expected regret in terms of order.
% Nevertheless, we remark that the agent lacks this oracle-like knowledge regarding the index of the best arm, and must learn it on the fly.

Theorem~\ref{theorem_impossible} demonstrates that the pursuit of identifying the optimal arm while minimizing the expected cumulative regret necessarily incurs a sample complexity on the order of $\omega (\log(1/\delta))$. This phenomenon can be understood through the inherent tension between learning the best arm and regret minimization in sequential decision-making.

On one hand, to reliably identify the best arm, the agent must gather sufficient statistical evidence to distinguish it from suboptimal alternatives. This fundamentally requires sampling all arms, including suboptimal ones, a number of times proportional to  $\log(1/\delta)$, in order to accumulate enough information to rule out incorrect arms. On the other hand, minimizing cumulative regret incentivizes the agent to pull suboptimal arms as infrequently as possible, and instead exploit the arm that appears to be the best.

Since pulling the optimal arm incurs zero regret, it is appealing for the agent to select it as frequently as possible to refine its estimate of the optimal mean. However, this creates an inefficiency from the perspective of BAI: repeatedly sampling the best arm yields diminishing returns in distinguishing it from competing arms, as the critical information lies in {\em comparisons} with suboptimal arms. As a result, even though the number of pulls of suboptimal arms is kept at the minimal logarithmic order, the agent must compensate by sampling the best arm significantly more often to satisfy the confidence requirement. This leads to a total sample complexity that is strictly larger than $\Theta(\log(1/\delta))$.

Importantly, the agent does not know the identity of the optimal arm a priori and must learn it adaptively from data. Thus, it is forced to operate under uncertainty, balancing exploration and exploitation without oracle knowledge. Theorem~\ref{theorem_impossible} formalizes the consequence of this tradeoff: any algorithm that achieves asymptotically optimal regret must necessarily incur a super-logarithmic number of total samples.

We now provide a slightly more technical proof sketch of Theorem \ref{theorem_impossible}, deferring the details to  Appendix~\ref{appendix_proof_theorem_impossible}. 

%\begin{proof}[Proof Sketch of Theorem \ref{theorem_impossible}]
\paragraph{Proof sketch of Theorem~\ref{theorem_impossible}.}
Assume there exists an asymptotically optimal $\delta$-PAC algorithm with 
$\mathbb{E}_{\bmu}[\tau_\delta] = O(\log(1/\delta))$. Then, by  asymptotic optimality (Definition \ref{definition_optimal}),
\[
\mathbb{E}_{\bmu}[R(\tau_\delta)] = \sum_{i>1} \Delta_i \mathbb{E}_{\bmu}[N_i(\tau_\delta)]
\sim \mathrm{I}^*(\bmu)\log(1/\delta),
\]
which implies
\[
\mathbb{E}_{\bmu}[N_i(\tau_\delta)] \sim \frac{\log(1/\delta)}{\kl(\mu_i,\mu_1)}, 
\quad \forall\, i>1.
\]
Thus, suboptimal arms are sampled $\Theta(\log(1/\delta))$ times. 

By the  change-of-measure argument \citep[Lemma 1]{kaufmann2016complexity},
\begin{align}
    \inf_{\blambda \in \mathrm{Alt}(\bmu)} \sum_{i=1}^K    \mathbb{E}_{\bmu}[N_i(\tau_\delta)]\cdot\kl(\mu_i,\lambda_i) 
\geq \kl_{\mathcal{B}}(\delta,1-\delta),\label{eqn:change_of_meas}
\end{align}
where $\mathrm{Alt}(\bmu)$ is the set of alternative instances where arm $1$
 is not the best arm. Fixing $j>1$ and restricting the infimum to alternative instances where only arms $1$ and $j$ are modified, we define the function 
$
\Phi(x,y) := \inf_{\lambda \in I} \big(x\,\kl(\mu_1,\lambda) + y\,\kl(\mu_j,\lambda)\big)
$.
Through the simplification of the infimum in~\eqref{eqn:change_of_meas}, we obtain
\[
\Phi\bigg(
\underbrace{\frac{\E_\mu[N_1(\tau_\delta)]}{\log(1/\delta)}}_{=: \ x},\,
\underbrace{\frac{1}{\kl(\mu_j\mu_1)}}_{=:\ y}
\bigg) \ge 1.
\]
However, for any finite $x$, we have $\Phi(x,y) < y \kl(\mu_j,\mu_1) = 1$, which forces $x \to +\infty$.  Hence 
$\E_\mu[N_1(\tau_\delta)] = \omega(\log(1/\delta))$, implying 
$\E_\mu[\tau_\delta] = \omega(\log(1/\delta))$.

% To conclude this section, we remark that a recent work by \citet{kanarios2024cost} considers the cumulative cost up to a stopping time for correctly identifying the best arm with probability at least $1 - \delta$. In contrast, Theorem \ref{theorem_impossible} adopts a different perspective: it establishes lower bounds on the expected number of arm pulls for any asymptotically optimal algorithm (per Definition \ref{definition_optimal}), with an emphasis on the regret up to the stopping time. This fundamental difference in objectives partly explains why the expected sample complexity in Theorem 5 is $\omega(\log(1/\delta))$, which exceeds the more common $\Theta(\log(1/\delta))$ rate typically seen in fixed-confidence best arm identification problems.

\section{The Double KL-UCB Algorithm}
\label{section_algo}
In this section, we introduce a conceptually simple algorithm, namely Double KL-UCB (or \textsc{DKL-UCB}), to identify the best arm with minimal regret. The pseudocode for \textsc{DKL-UCB} is presented in Algorithm~\ref{algo1} and further elucidated in the subsequent discussion.

\begin{algorithm}[t]
	\caption{Double KL-UCB (or \textsc{DKL-UCB})}
	\label{algo1}
	\hspace*{0.02in} {\bf Input:} Arm set $[K]$ and confidence level $\delta \in (0, 1).$
	\begin{algorithmic}[1]
	\STATE Sample each arm once, and set $t= K$.
        \REPEAT
        \STATE Update $\hat \mu_i (t)$ and $N_i(t)$ for all $ i \in [K]$, and increment $t\leftarrow t+1$.
        \STATE For each arm $i\in[K]$, compute the quantities
        \begin{equation}
        \label{equation_ucbf}
        U^f_i(t)=\sup \left\{\mu \in I: {\mathrm{kl}\left(\hat{\mu}_i(t-1), \mu\right)}  \leq \frac{f(t)}{N_i(t-1)}\right\}
        \end{equation}
        and 
        \begin{equation}
        \label{equation_ucbg}
        U^g_i(t)=\sup \left\{\mu \in I: {\mathrm{kl}\left(\hat{\mu}_i(t-1), \mu\right)}  \leq \frac{g(\delta,t)}{N_i(t-1)}\right\}.
        \end{equation}
        \STATE Let 
        \begin{equation}
        A_{t}^{f}=\argmax_{i\in[K]} U^f_{i}(t)\quad \mbox{and} \quad A_{t}^{g}=\argmax_{i\in [K]\setminus \{A_t^f\}} U^g_{i}(t).\label{eqn:defAs} 
        \end{equation}
        \STATE Flip a coin with bias (probability of heads)  $\beta(\delta)$, as defined in~\eqref{equation_beta}. 
        \STATE If the outcome is heads, then sample $A_t = A_t^f$; otherwise, sample $A_t = A_t^g$.
        \UNTIL{$L^g_{A_{t}^f}(t)> U^g_{A_t^g}(t)$} \hfill \COMMENT{See the definition of $L_i^g(t)$ in~\eqref{equation_lcbg}}
	\end{algorithmic}
	\hspace*{0.02in} {\bf Output:}  The candidate best arm $i_{\mathrm{out}} = A_t^f$.
\end{algorithm}

As its name suggests, our algorithm \textsc{DKL-UCB} leverages two types of Kullback--Leibler upper confidence bounds (UCBs) to guide arm sampling decisions. During the initialization phase, each arm is sampled exactly once. Subsequently, at each time $t$, we compute two UCBs, $U^f_i(t)$ and $U^g_i(t)$, for each arm $i\in[K]$. These UCBs, defined in Equations~\eqref{equation_ucbf} and~\eqref{equation_ucbg}, are  indexed by exploration functions:
\begin{equation*}
f(t) = 3 \log t \quad \text{and} \quad g(\delta, t) = \log\left(\frac{2Kt^2}{\delta}\right).
\end{equation*}
For ease of presentation, we refer to $U^f_i(t)$ and $U^g_i(t)$ as $f$-UCB and $g$-UCB, respectively. 
In our algorithm as well as its analysis, two candidate arms are of particular significance: let $A_{t}^{f}$ denote the arm with the highest $f$-UCB, and $A_{t}^{g}$ represent the arm with the highest $g$-UCB excluding $A_{t}^{f}$; see Equation~\eqref{eqn:defAs}. We then proceed to sample either $A_{t}^{f}$ or $A_{t}^{g}$ in a randomized fashion. Specifically, we toss a biased coin, with the probability of landing heads given by
\begin{equation}
\label{equation_beta}
\beta(\delta) = 1 - \min\left\{ \frac{1}{\log \log(1/\delta)}, \frac{1}{2} \right\}.
\end{equation}
If the outcome is heads, we pull the arm $A_t = A_t^f$; otherwise, we pull $A_t = A_t^g$.

Due to its objective of best arm identification in the fixed-confidence setting, our algorithm needs to actively stop when the confidence level $\delta$ is reached. For the stopping rule, we utilize a variant of the lower confidence bound, based on the exploration function $g(\delta, t)$. In particular, for each arm $i\in[K]$, we define its $g$-LCB as follows:
\begin{equation}
\label{equation_lcbg}
L^g_i(t)=\inf \left\{\mu \in I: {\mathrm{kl}\left(\hat{\mu}_i(t-1), \mu\right)}  \leq \frac{g(\delta,t)}{N_i(t-1)}\right\}.
\end{equation}
Ultimately, our algorithm terminates when the $g$-LCB of $A_{t}^f$ exceeds the $g$-UCB of $A_{t}^g$. In other words, it stops at the stopping time
\begin{equation}
\label{equation_stoppingrule}
    \tau_\delta = \inf\left\{t \in \mathbb{N} : L^g_{A_{t}^f}(t) > U^g_{A_t^g}(t)\right\},
\end{equation}
and recommends $A_{\tau_\delta}^f$ as the identified best arm $i_{\mathrm{out}}$.

\begin{remark} \label{rmk:dkl_ucb}
It is worth discussing some similarities between our algorithm \textsc{DKL-UCB} and other algorithms in the multi-armed bandit literature. If our algorithm consistently opts to pull the arm $A_{t}^{f}$, then its sampling rule  mirrors that of  \textsc{KL-UCB}, a celebrated algorithm for cumulative regret minimization \citep{cappe2013kullback}, up to the choice of the exploration function $f(t)$. 

Moreover, concerning the randomized dynamics within the sampling rule, our algorithm shares commonalities with Top-Two Thompson Sampling (\textsc{TTTS}), which is designed for fixed-confidence best arm identification \citep{russo2020simple}. Specifically, both algorithms incorporate randomness in selecting one of two candidate arms. However, apart from how the two candidates (called leader and challenger in the \textsc{TTTS} literature)  are determined, another significant distinction lies in the selection probability. In \textsc{TTTS}, each candidate is pulled with a fixed probability, which is independent of the confidence level $\delta$. In contrast, in our algorithm, as $\delta$ approaches zero, the coin bias $\beta(\delta)$ tends toward one. Consequently, the arm $A_{t}^{f}$ dominates the proportion of pulls. In this respect, the alternative candidate $A_{t}^{g}$ serves as a minor yet critical supplement to the arm sampling rule.
\end{remark}

\section{Theoretical Analysis of {\sc DKL-UCB}}
\label{section_theory}

In this section, we theoretically analyze the performance of our algorithm \textsc{DKL-UCB} from multiple perspectives. Our main results are presented in Theorem~\ref{theorem_upper} below. The complete proof of Theorem~\ref{theorem_upper} are deferred to Appendix~\ref{appendix_proof_theorem_upper}. In Section~\ref{section_theory_discussion}, we discuss the primary technical challenges and provide a proof sketch of Theorem~\ref{theorem_upper}.

\subsection{Main Results}
\begin{theorem}
\label{theorem_upper}
For every confidence level $\delta \in (0, 1)$,  the \textsc{DKL-UCB} algorithm (Algorithm~\ref{algo1}) has the following properties. For all bandit instances $\bmu \in \mathcal{M}$, it guarantees that
\begin{align*}
    \PP(i_{\rm{out}}\neq 1)  &\leq \delta.  \qquad \tag{\textsc{DKL-UCB} is $\delta$-PAC} 
    \end{align*}
 Furthermore, when $\delta\to0$,   the regret of  \textsc{DKL-UCB} satisfies
 \begin{align*}
\limsup_{\delta\to 0} \frac{\E_{\bmu}[ R(\tau_\delta)]} {\log (1/\delta)}& \le \mathrm I^{*}(\bmu),  \qquad \qquad \tag{Regret Bound}
\end{align*}
and
\begin{align*}
\lim_{\delta\to 0}  \frac {\E_{\bmu}[\tau_\delta]} {\log (1/\delta)\cdot (\log \log(1/\delta))^2 } = 0. \tag{Sample Complexity}
\end{align*}
\end{theorem}

Theorem~\ref{theorem_upper} first guarantees the correctness of our algorithm: the probability of recommending a suboptimal arm is no more than $\delta$. 
As for the expected cumulative regret, by comparing the instance-dependent upper bound in Theorem~\ref{theorem_upper} with the corresponding lower bound in Theorem~\ref{theorem_lowerbound}, it is evident that our algorithm \textsc{DKL-UCB} is asymptotically optimal for the problem of BAI with minimal regret. Particularly, for all bandit instances $\bmu \in \mathcal{M}$, the expected regret of our algorithm satisfies the following limiting behaviour:
\begin{equation*}
\lim_{\delta\to 0} \frac{\E_{\bmu}[ R(\tau_\delta)]} {\log (1/\delta)} = \mathrm I^{*}(\bmu).
\end{equation*}

Furthermore, regarding the expected sample complexity,  in conjunction with the impossibility result in Theorem~\ref{theorem_impossible}, Theorem~\ref{theorem_upper} shows that our algorithm satisfies 
\begin{equation}
    \mathbb{E}_{\bmu}[\tau_{\delta}] = \omega({\log (1/\delta)}) \cap o \left(\log (1/\delta)\cdot (\log \log(1/\delta))^2\right).\label{eqn:samp_comp}
\end{equation}
Even though the principal performance metric for our problem is not the expected sample complexity, the result  in~\eqref{eqn:samp_comp} demonstrates that the sample complexity of our algorithm {\sc DKL-UCB} is close to the fundamental barrier in Theorem~\ref{theorem_impossible}, up to a small multiplicative factor of the order of $(\log \log(1/\delta))^2$. Consequently, our algorithm is nearly optimal in terms of sample complexity while achieving optimality in the cumulative regret.

%Even though our main objective is not to quantify the expected sample complexity of an asymptotically optimally algorithm (for the problem of BAI with minimal regret), we observe that we have identified the expected sample complexity up to a small multiplicative factor. 

%\textcolor{red}{I think it's important, in relation to the comments in  Remark 2, to add some comments on why the algorithm is designed such that most of the time, $A_t^f$ is pulled and why there is a necessity to also, with small probability, pull arm $A_t^g$.}

\subsection{Technical Challenges and Proof Outline}
\label{section_theory_discussion}

{In this section, we discuss the key technical challenges and outline the proof of Theorem \ref{theorem_upper}. Readers primarily interested in the theoretical and experimental results may choose to skip this subsection on their first reading and proceed directly to Section \ref{section_compare_other}.}

Our instance-dependent lower bound in Theorem~\ref{theorem_lowerbound} suggests that the optimal regret upper bound could be $\mathrm{I}^*(\bm{\mu})\log(1/\delta)+o(\log(1/\delta))$ as the confidence level $\delta$ tends to zero. In contrast, the optimal regret bound for cumulative regret minimization is $\mathrm{I}^*(\bm{\mu})\log(T)+o(\log T)$ as the time horizon $T$ tends to infinity, which indicates that an optimal regret minimization algorithm pulls each suboptimal arm $O(\log T)$ times over time $T$ (see \eqref{equation_asymptotic_lowerbound_regretmin} in Section~\ref{section_compare_other}). Thus, one natural strategy for our problem is to run the \textsc{KL-UCB} algorithm and check whether the best arm can be reliably identified over time. However, the sample complexity of such an algorithm is easily seen to be $\Theta(1/\delta)$, which is significantly greater than $\Theta(\log(1/\delta))$. 

To address this problem, we introduce {\sc DKL-UCB}, which utilizes {\em two} UCB indices, $f$-UCB and $g$-UCB, as defined in  \eqref{equation_ucbf} and \eqref{equation_ucbg}, respectively. In our approach, $f$-UCB, with $f(t) = 3\log t$, is primarily employed for exploring the best arm. This choice differs from the standard \textsc{KL-UCB} algorithm~\citep{cappe2013kullback} where the exploration function $f(t)$ is approximately $\log t$. 
While one might suspect that $f(t) = 3\log t$ could result in suboptimal asymptotic regret for the regret minimization task due to the potential non-optimality of the constant $3$, the inflated choice of $f(t)$ does not lead to suboptimal regret for our purpose.
Roughly speaking, this is because the stopping time $\tau_{\delta}$ satisfies $\mathbb{E}_{\bmu}[\tau_\delta] = O(\log^2 (1/\delta))$, and hence the regret of $f$-UCB (when $A_t = A_t^f$) within $O(\log^2 (1/\delta))$ time steps is of order $o(\log (1/\delta))$.
Therefore, the inflated choice of $f(t)$ does not hinder the asymptotic optimality properties of  {\sc DKL-UCB} for the unique problem under consideration.

% As indicated from our lower bound, the stopping time $\tau_{\delta}$ satisfies $\mathbb{E}_{\bmu}[\tau_\delta] = \omega(\log (1/\delta))$. In the meanwhile, we aim to achieve $\mathbb{E}_{\bmu}[\tau_\delta] = O(\log^2 (1/\delta))$, and hence the regret of $f$-UCB within $O(\log^2 (1/\delta))$ steps is of order $o(\log (1/\delta))$. 

\begin{figure}[h]
\begin{picture}(450,180)
\put(430,110){$t$}
\put(0,100){\vector(1,0){430}}
\put(5,95){\line(0,1){10}}
\put(5,87){$0$}
\put(100,95){\line(0,1){10}}
\put(100,87){$\psi$}
\put(200,95){\line(0,1){10}}
\put(200,85){$T_0$}
\put(300,95){\line(0,1){10}}
\put(300,87){$t_0$}
\put(320,95){\line(0,1){10}}
\put(320,87){$t_1$}
\put(350,95){\line(0,1){10}}
\put(350,87){$t_2$}
\put(362,87){$\cdots$}
\put(400,95){\line(0,1){10}}

%\put(100,95){\line(1,1){10}}
\multiput(201,95)(5,0){19}{\line(1,1){10}} 

\put(100,150){\vector(0,-1){40}}
\put(80,172){$\forall\, t\ge\psi$}
\put(80,157){$U_1^f\ge\mu_1-\epsilon$}

\put(200,150){\vector(0,-1){40}}
\put(162,172){$\forall\, t\ge T_0$}
\put(162,157){$L_1^g\ge\mu_1-\epsilon$}

\put(400,150){\vector(0,-1){40}}
\put(400,87){$t_r$}
\put(355,172){\textsc{DKL-UCB} stops here}
\put(355,157){w.p.\ $\ge 1-t_r^{-2}$}

\put(255,150){\vector(0,-1){47}}
\put(235,172){With high prob.}
\put(235,157){\textsc{DKL-UCB} stops here}

\put(30,30){\vector(1,1){60}}
\put(0,18){$\mathbb{E}[\psi]=o\big(\log(1/\delta)\big)$}

\put(130,30){\vector(1,1){60}}
\put(110,18){$T_0\approx\log(1/\delta))\big(\log\log(1/\delta)\big)^{3/4}$}
\put(110,3){$|\{t:A_t=A_t^g\}|= o\big(\log(1/\delta)\big)$}
\end{picture}
\caption{{Illustration of the proof.}}
\label{fig:proof_illus}
\end{figure}

{In our analytical framework, we partition the time horizon into intervals of exponentially growing lengths; see Figure~\ref{fig:proof_illus} for an illustration.} Let $\gamma = \log\log (1/\delta)$, $\epsilon=\gamma^{-1/4}$ and $h(\delta)=\log (1/\delta)\cdot \gamma/\epsilon= \log (1/\delta)\cdot (\log \log(1/\delta))^{5/4}$. Then the partition is defined by the time points  $t_r = 2^{r}h(\delta)$, with each subinterval taking the form $(t_r, t_{r+1}]$, and an additional initial interval $(0, t_0]$.  Since the best arm is pulled for a sufficient number of times through $f$-UCB, we can show that $\PP(L_{1}^{g}(t_r)\geq \mu_1-\epsilon)\geq 1-1/t_{r}^2$. In view of this, to check that the stopping rule in~\eqref{equation_stoppingrule} is fulfilled, it suffices to additionally check that, with high probability, %Thus, to validate the condition~\eqref{equation_stoppingrule} in the stopping rule, we need to demonstrate that, with high probability,
$U_{i}^{g}(t_r) < \mu_1 - \epsilon$ for all suboptimal arms $i\in[K]\setminus \{1\}$.

To speed up the exploration of suboptimal arms, we incorporate supplementary randomized exploration using $g$-UCB, which is specifically tailored for the suboptimal arms. Notably, the exploration function $g(\delta, t)$ associated with $g$-UCB {depends on} the confidence level $\delta$. In our algorithm, we exclude $A_t^f$, the arm with the highest $f$-UCB, when determining $A_t^g$, and pull the arm $A_t^g$ with a probability of $1- \beta(\delta)=\min\{1/\gamma,1/2\}$. At time $t_r$, the expected number of times that we pull the arm $A_{t}^{g}$ is $t_r (1-\beta(\delta))=\omega(2^{r}\log (1/\delta))$, which   ensures that  $\max_{i>1} U_{i}^{g}(t_r)< \mu_1-\epsilon$ with probability at least $1-1/t_{r}^2$. Furthermore, given that $L_{1}^{g}(t_r) \ge \mu_1 - \epsilon$, the probability that the algorithm terminates after $t_r$ is at most $1/t_r^2$. 
This proposition is formally established in Lemma~\ref{lem:I1-5}. In its proof detailed in Appendix~\ref{subsection_important_lemma}, we introduce four events that are {\em a priori} dependent, but we   meticulously decouple their complex interdependencies.
Consequently, the regret incurred from pulling arms beyond time $t_0$ can be bounded such that optimality is ensured.

{Here, we provide further insights into the role of the parameter $\beta(\delta)$. This parameter is crucial in accelerating the learning process and thereby reducing the sample complexity. As $\delta$ approaches 0, $\beta(\delta)$ approaches 1, indicating that the algorithm \textsc{DKL-UCB} increasingly favors $A_t^f$ over $A_t^g$. This aligns with the lower bound intuition, where: (1) to achieve asymptotic optimality, the optimal arm must be selected $\omega(\log(1/\delta))$ times; and (2) for large values of $t$, when $A_t = A_t^f$, the algorithm primarily chooses the optimal arm. Conversely, with probability $1 - \beta(\delta)$, the algorithm selects $A_t = A_t^g$, which is likely a suboptimal choice for large $t$. Therefore, if $t$ exceeds $\log(1/\delta)\cdot(\log \log(1/\delta))^2$, \textsc{DKL-UCB} will, with high probability, select suboptimal arms at least $\log(1/\delta)\cdot(\log \log(1/\delta))^2 \cdot (1 - \beta(\delta)) = \omega(\log(1/\delta))$ times. This guarantees sufficient exploration of suboptimal arms before reaching the time $ \log(1/\delta)\cdot(\log \log(1/\delta))^2$.}

To  bound the total expected cumulative regret, a key theoretical challenge arises from the regret incurred before $t_0$, i.e., $\E_{\bmu}[ R(t_0)]$.  Since $t_0 = \omega(\log (1/\delta))$, the regret bound of $\E_{\bmu}[ R(t_0)]$ could be suboptimal if not appropriately managed. Let $\psi$ represent the smallest time step $t$ such that $U_{1}^{f}(t)\geq \mu_1-\epsilon$, i.e., $\psi:=\inf\{ t: U_{1}^{f}(t)\geq \mu_1-\epsilon\}.$  
By leveraging a novel concentration inequality for exponential families (Lemma \ref{lem:Rt0-1} in Appendix \ref{appendix_proof_theorem_upper}), we can establish that $\EE_{\bmu}[\psi]=O(1/\epsilon^2) = o(\log (1/\delta))$.

Our analysis relies on a carefully chosen time step $T_0$, ensuring that, with high probability, the optimal arm is well-estimated at time $T_0$ (i.e., $L_{1}^{g}(T_0) \ge \mu_1 - \epsilon$) and the regret from suboptimal arms remains within an acceptable range. Specifically, we define $T_0$ as follows:
\begin{align}
T_0 = \min \Bigg\{ \psi + \underbrace{\sum_{i=2}^K \sum_{t=\psi+1}^{t_0} \mathbf{1} \left\{ A_t = i, \mathrm{kl}(\hat{\mu}_i(t-1), \mu_1 - \epsilon) \leq \frac{f(t_0)}{N_i(t-1)} \right\}}_{Q} \, + \, \gamma \epsilon \log \frac{1}{\delta}, \,  t_0 \Bigg\}. \nonumber
\end{align}
Here we consider the nontrivial case where $T_0 < t_0$. Since $T_0$ exceeds $\psi$, we have $U_{1}^f(T_0) > \mu_1 - \epsilon$. Besides, according to the dynamics of our algorithm, the term $Q$ represents an upper bound on the number of pulls of the suboptimal arms through $f$-UCB. That is, $Q$ bounds the number of times $A_t =A_t^f \in [K] \setminus \{1\}$ for $t$ between $\psi+1$ and $t_0$. Therefore, the $\gamma \epsilon \log (1/\delta)$ time steps within the definition of $T_0$ encompass either pulls of the optimal arm or the time steps when $A_t = A_{t}^{g}$. 
Utilizing a concentration bound, we can deduce that the number of times when $A_t = A_{t}^{g}$ within the $\gamma \epsilon \log (1/\delta)$ time steps is at most $ 2\epsilon \log (1/\delta)$, with high probability.
As a result, by time $T_0$, the algorithm will have pulled the optimal arm a sufficient number of times such that $L_{1}^{g}(T_0) \ge \mu_1 - \epsilon$ with  probability   at least $1-1/t_0$.

Finally, the number of pulls of the suboptimal arms before time $T_0$ is of order $O(\psi+Q+\epsilon\log(1/\delta)) = o(\log(1/\delta))$, which does not result in suboptimal regret. Additionally, Lemma~\ref{lem:Rt0-4} demonstrates that from $T_0+1$ to $t_0$, the number of pulls of any suboptimal arm $i \in [K]\setminus\{1\}$ does not exceed 
% \[
% \max\bigg\{\frac{g(\delta, t_0) }{\mathrm{kl}(\mu_i+\epsilon,\mu_1-\epsilon)},\frac{2V\log (t_0)}{\epsilon^2}\bigg\} =  \frac{\log (1/\delta) }{\mathrm{kl}(\mu_i,\mu_1)} + o\bigg(\log \frac{1}{\delta}\bigg)
% \] 
$g(\delta, t_0)/\mathrm{kl}(\mu_i + \epsilon, \mu_1 - \epsilon)$
with a probability of at least $1-1/t_0$. Altogether, we can bound the regret incurred before $t_0$ optimally.

\vspace{3pt}
\section{Discussion: Comparisons to Related Problems}
\label{section_compare_other}
In this following, we compare our investigated problem---best arm identification with minimal regret---with two other extensively studied problems in the bandit literature, focusing on their respective asymptotic performances.

\paragraph{Cumulative regret minimization. } The objective of this problem is to maximize the expected cumulative rewards up to the time horizon $T \in \mathbb N$, which might be either known or unknown to the agent a priori. Equivalently, the agent aims at minimizing the expected cumulative regret $\mathbb E_{\bmu}[R(T)]$.

In the seminal work by \citet{lai1985asymptotically}, it was proven that any \emph{consistent}\footnote{A regret minimization algorithm is said to be consistent, if for all bandit instances $\bmu \in \mathcal{M}$ and  all $\alpha>0$, ${\E_{\bmu}[ R(T)]}=o(T^\alpha)$.} regret minimization algorithm must satisfy that for all bandit instances $\bmu \in \mathcal{M}$,
\begin{equation}
\label{equation_asymptotic_lowerbound_regretmin}
\liminf_{T\to \infty} \frac{\E_{\bmu}[ R(T)]} {\log T} \ge \mathrm I^{*}(\bmu),
\end{equation}
where  $\mathrm I^{*}(\bmu)$ is the same as the   quantity  defined in Equation~\eqref{equation_Istar}.

Although BAI with minimal regret and cumulative regret minimization represent distinctly different tasks, the coincidence regarding $\mathrm I^{*}(\bmu)$ suggests profound connections between them. Specifically, the asymptotic lower bound presented in \eqref{equation_asymptotic_lowerbound} for BAI with minimal regret follows from the $\delta$-PAC nature of the algorithm, while the asymptotic lower bound in \eqref{equation_asymptotic_lowerbound_regretmin} relies on the assumption of consistency. This coincidence might be attributed to the shared consideration of cumulative regret in both problems, and represents an optimal equilibrium of regret among suboptimal arms. We defer a more in-depth investigation into this intriguing phenomenon as future work. In the current study, this insightful finding is leveraged in the design of our algorithm; see Section~\ref{section_theory_discussion}.

% \begin{equation*}
%     \min \  \mathbb E_{\bmu}[R(T)] .
% \end{equation*}

\paragraph{Best arm identification with minimal samples. } 
In contrast to our problem setup, this problem shares the common goal of identifying the best arm with a prescribed confidence level $\delta\in(0,1)$, but focuses on a distinct performance metric---the expected sample complexity $\mathbb E_{\bmu}[\tau_\delta]$. Using our  notation, the problem can be formally expressed as follows:
\begin{equation}
\label{equation_setup_minimalsample}
\begin{aligned}
    \min \quad  &\mathbb E_{\bmu}[\tau_\delta] \\
    \text{s.t.}  \quad &\mathbb P _{\bmu}(\tau_\delta<\infty) = 1 \text{ and } \mathbb P _{\bmu} ( i_{\mathrm{out}} \neq 1 ) \le \delta.
\end{aligned}
\end{equation}
It is worth highlighting that the fundamental distinction between \eqref{equation_setup} and \eqref{equation_setup_minimalsample} lies solely in the choice of the objective function.

The instance-dependent lower bound for the problem of BAI with minimal samples can be characterized through a max-min optimization problem \citep{garivier2016optimal}. Specifically, let $\mathrm{Alt}(\bmu)$ denote the set of alternative bandit instances where arm $1$ is not the best arm. Then for any $\delta$-PAC BAI algorithm, it holds that
\begin{equation*}
\label{equation_asymptotic_lowerbound_minimalsample}
\liminf_{\delta\to 0} \frac{\E_{\bmu}[ \tau_\delta]} {\log (1/\delta)} \ge \Gamma^{*}(\bmu)
\end{equation*}
where
\begin{equation}
\label{equation_asymptotic_lowerbound_minimalsample_maxmin}
\Gamma^{*}(\bmu)^{-1} := \sup _{w \in \mathcal P _{K}} \inf _{\blambda \in \mathrm{Alt}(\bmu)}\left(\sum_{i=1}^K{w_i \mathrm{kl} \left(\mu_{i}, \lambda_{i}\right)}\right).
\end{equation}
Note that this lower bound can be attained asymptotically by approximating the optimal proportion of arm pulls, as indicated by the outer supremum over ${w \in \mathcal P _{K}}$ in Equation~\eqref{equation_asymptotic_lowerbound_minimalsample_maxmin} \citep{garivier2016optimal, degenne2019non,mukherjee2023best}. 

A natural question emerges when we consider these two problems in further depth: Does an optimal policy for BAI with minimal samples inherently lead to a commensurately low cumulative regret? The following example provides a counterexample to this assertion, underscoring the significance of carefully investigating the problem of BAI with minimal regret.

\begin{example} \label{eg:two_armed}
Consider a two-armed Bernoulli bandit instance $\bmu = (1-\mu, \mu)$ with $\mu \in (0, 1/2)$. Here, as with the rest of the paper, we focus on the asymptotic regime that $\delta$ tends to zero. For BAI with minimal samples, a simple mathematical derivation reveals that the optimal proportion of arm pulls is uniform (see Appendix~\ref{app:two_armed}), and hence the expected cumulative regret under this scenario is approximately $\frac {(1-2\mu) \log (1/\delta)} {2 \mathrm{kl}_{\mathcal B}(\mu, 1/2)}$. On the other hand, for the optimal policy of BAI with minimal regret, its expected cumulative regret is given by $\frac {(1-2\mu) \log (1/\delta)} { \mathrm{kl}_{\mathcal B}(\mu, 1-\mu)}$. As the parameter $\mu$ approaches zero, the ratio between these two regret values, $\frac  { \mathrm{kl}_{\mathcal B}(\mu, 1-\mu)} {2 \mathrm{kl}_{\mathcal B}(\mu, 1/2)} \simeq \frac {\log (1/\mu)} {2\log 2}$, can become arbitrarily large.
\end{example}

\section{Numerical Experiments}
\label{section_experiment}
\begin{figure}[t]
 \centering
	\begin{minipage}{0.9\linewidth}
		\subfigure[Cumulative regret for instance $\bmu^{\dagger}$.]{
			\begin{minipage}[b]{0.5\textwidth}
				\includegraphics[width=1.000\textwidth]{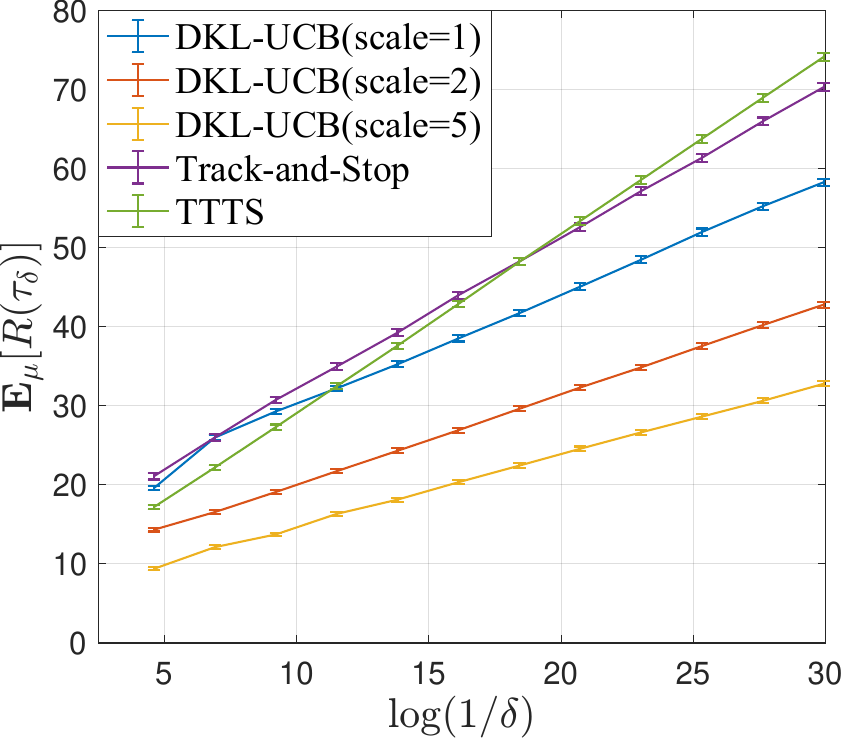}
			\end{minipage}
		}
		\hspace{0.00pt}
		\subfigure[Sample complexity for instance $\bmu^{\dagger}$.]{
			\begin{minipage}[b]{0.5\textwidth}  
				\includegraphics[width=1.000\textwidth]{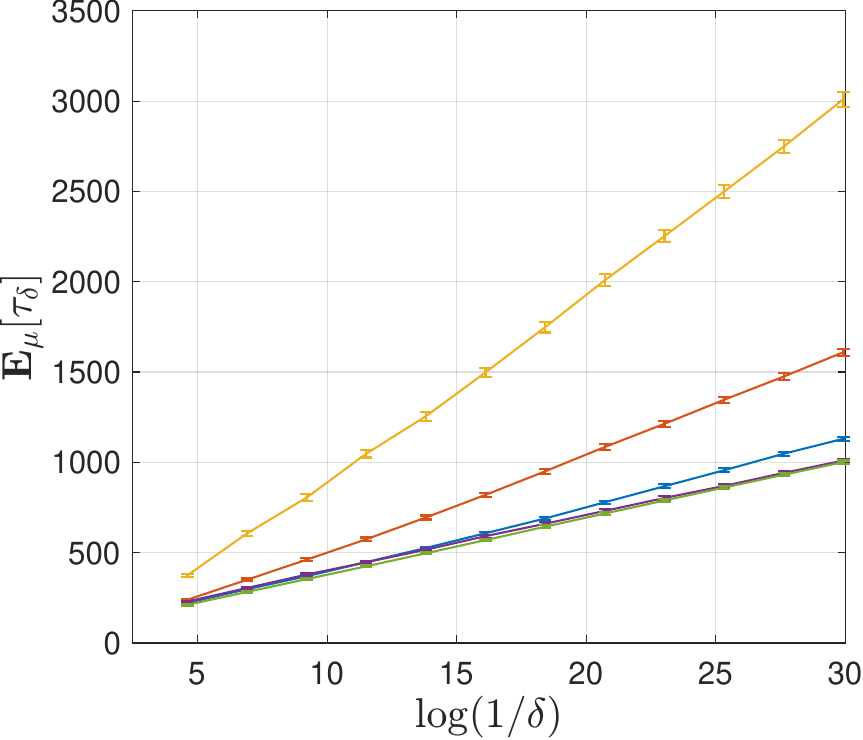}
			\end{minipage}
		}
            
    		\subfigure[Cumulative regret for instance $\bmu^{\ddagger}$.]{
    			\begin{minipage}[b]{0.5\textwidth}
    				\includegraphics[width=1.000\textwidth]{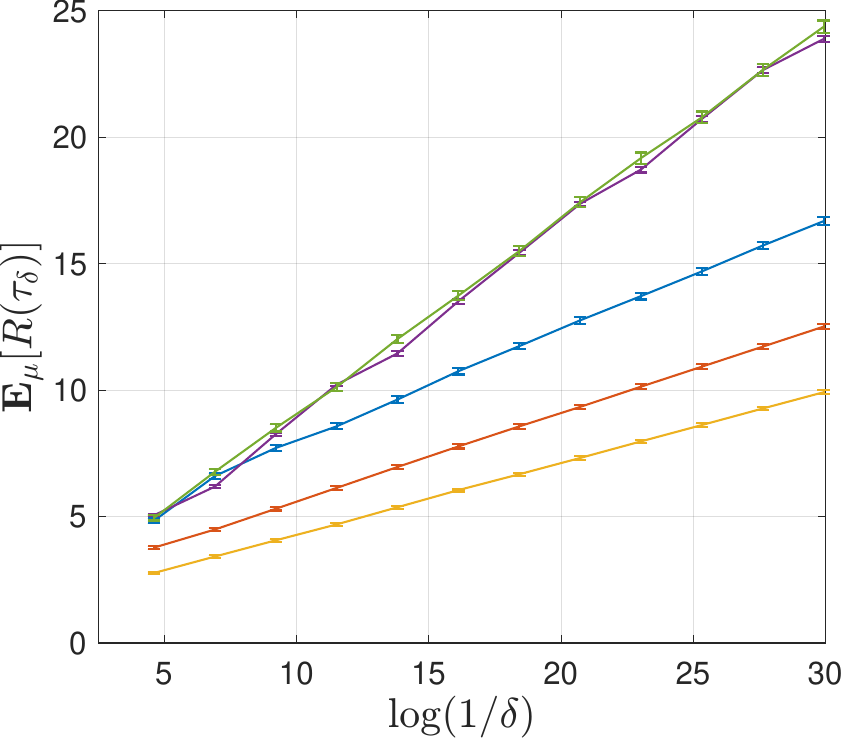}
    			\end{minipage}
    		}
    		\hspace{0.00pt}
    		\subfigure[Sample complexity for instance $\bmu^{\ddagger}$.]{
    			\begin{minipage}[b]{0.5\textwidth}  
    				\includegraphics[width=1.000\textwidth]{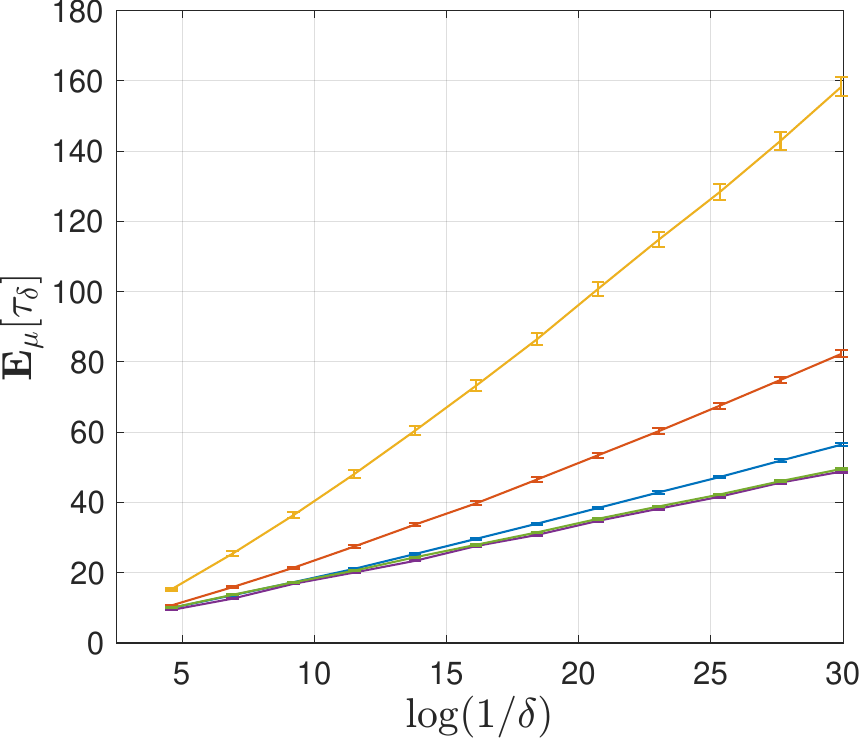}
    			\end{minipage}
    		}
	\end{minipage}
        % \vspace{-5pt}
	\caption{Empirical cumulative regrets  (up to stopping times) and sample complexities for different confidence levels $\delta$.}
	\label{figure_plots}
 \vspace{0pt}
\end{figure}

To validate our theoretical results, we compare our {\sc DKL-UCB} algorithm with two established approaches: Track-and-Stop \citep{garivier2016optimal} and Top-Two Thompson Sampling ({\sc TTTS}) \citep{russo2020simple}. To guarantee fair comparisons, all algorithms employ the  generalized likelihood ratio test (GLRT) stopping rule introduced in \citet{garivier2016optimal}. We note that standard \textsc{KL-UCB} \citep{garivier2011kl, cappe2013kullback} never terminates under this stopping rule and is therefore excluded from the comparison.\footnote{{Specifically, the threshold function used in the GLRT stopping rule grows too rapidly for standard \textsc{KL-UCB}  to satisfy the stopping criterion.}}

In our experiments, we consider a family of {\sc DKL-UCB} algorithms parameterized by a scaling factor $c>0$ within the definition of $\beta(\delta)$ in~\eqref{equation_beta}. Specifically, we modify $\beta(\delta)$ to:
\begin{align}
   \beta_c(\delta) = 1 - \min\left\{ \frac{1}{c\ \log \log(1/\delta)}, \frac{1}{2} \right\}. \label{eqn:modified_beta} 
\end{align}
When \( c=1 \), \( \beta_c(\delta) \) simplifies to the original form of \( \beta(\delta) \) given in \eqref{equation_beta}. As \( c \to +\infty \), $\beta_c(\delta)\to 1^-$ and the arm \( A_t^f \), defined in \eqref{eqn:defAs}, is pulled more frequently, making the algorithm increasingly resemble {\sc KL-UCB} as mentioned in Remark~\ref{rmk:dkl_ucb}. We note, however, that regardless of the exact value of $c >0$, our proposed {\sc DKL-UCB} algorithm possesses the same theoretical asymptotic guarantees as described in Theorem~\ref{theorem_upper}. 

We consider two problem instances: a five-armed and a two-armed Bernoulli bandit with mean vectors
$$\bmu^{\dagger} = [1, 0.9, 0.8, 0.7, 0.6]\quad  \text{and} \quad \bmu^{\ddagger} = [0.99, 0.01].$$
The latter instance (i.e., $\bmu^{\ddagger}$) corresponds to the setting described in Example~\ref{eg:two_armed} where $\mu=0.01$ is small. This configuration is expected to highlight a divergence in the performances of {\sc DKL-UCB} and Track-and-Stop. For each setting, the results are averaged over $1000$ independent trials with standard errors displayed as error bars.
Our results are depicted in Figure~\ref{figure_plots}. The top and bottom rows display results for instances $ {\bmu}^{\dagger}$   and ${\bmu}^{\ddagger}$, respectively, while the left and right columns present empirical regrets and sample complexities.

From Figures~\ref{figure_plots}(a) and (c), we observe that the regrets of {\sc DKL-UCB} are often smaller compared to those of Track-and-Stop and {\sc TTTS}. This outcome aligns with expectations, as the latter algorithms are specifically designed to achieve asymptotically optimal sample complexity for fixed-confidence BAI and not to minimize cumulative regret. When considering larger values of $\delta$, {\sc DKL-UCB} with $c=1$  does not consistently outperform its competitors. However, this limitation can be addressed by selecting a larger scaling factor (such as $c=5$), indicating that for moderate values of $\delta$, a more aggressive sampling approach of $A_t^f$ relative to $A_t^g$ is beneficial for reducing regret.  In our main theorem, we show that as $\delta \to 0$, choosing $\beta(\delta) \to 1$ ensures asymptotic optimality of our algorithm. For finite $\delta$, our experimental results suggest that this choice may also lead to good finite-time regret performance. We leave a more thorough investigation of this for future work.
%This insight aligns with the understanding that {\sc KL-UCB} is asymptotically optimal in terms of minimizing cumulative regret. 
Furthermore, as seen in Figure~\ref{figure_plots}(c), the slope of {\sc DKL-UCB} (across various scaling factors) is considerably smaller than that of Track-and-Stop and {\sc TTTS}. This observation supports our theoretical finding in Example \ref{eg:two_armed}, highlighting the significant regret difference between {\sc DKL-UCB} and algorithms designed for BAI with minimal samples.

Conversely, Figures~\ref{figure_plots}(b) and (d) show that Track-and-Stop and {\sc TTTS} outperform {\sc DKL-UCB} in terms of sample complexity. This is also expected, as these algorithms are specifically tailored for efficient BAI with minimal samples.

Overall, the experimental results corroborates our theoretical findings: {\sc DKL-UCB} excels at BAI with minimal regret. It outperforms baseline algorithms not optimized for this balancing act, highlighting the inherent trade-off between minimizing regret and sample complexity in the study of multi-armed bandits.  

\section{Conclusions and Future Work}
\label{section_conclusion}
In this paper, we investigate the problem of fixed-confidence best arm identification while minimizing the expected cumulative regret. By analyzing the instance-dependent lower bound, we clearly demonstrate the differences between this problem and other problems in the bandit literature. Concurrently, we design and analyze Double KL-UCB (or \textsc{DKL-UCB}), which stands out for its utilization of dual upper confidence bounds. Our rigorous theoretical examination firmly confirms its asymptotic optimality in terms of cumulative regret as well as its near-optimality in sample complexity.

While our algorithm \textsc{DKL-UCB} is applicable for any fixed confidence level $\delta\in (0,1)$, its theoretical guarantees are asymptotic. In our analysis, we leverage asymptotic techniques to streamline the proof process and presentation. Although non-asymptotic results are possible with our current proof strategies, they would be considerably more complex.  We leave a more concise non-asymptotic performance analysis for future work.

Moreover, a fruitful direction for future research lies in establishing bounds for the  worst case regret for our problem. Note that for the problem of cumulative regret minimization, the two predominant performance metrics are asymptotic and worst case regrets. In particular, the optimal bound for the latter is $O(\sqrt{KT})$. While our work studies the asymptotic regret performance for fixed-confidence BAI, the exploration of worst case regret remains open. %Can we develop an algorithm capable of attaining a worst-case regret of $O(\sqrt{K/\delta})$  across all possible bandit instances?

{An intriguing open question arises regarding the fundamental limits of sample complexity, denoted as \(\EE[\tau_\delta]\), as \(\delta\) approaches zero. From Definition~\ref{definition_optimal}, we understand the concept of asymptotic optimality and, as per Theorem~\ref{theorem_impossible}, any asymptotic optimal algorithm inherently incurs a sample complexity \(\EE[\tau_\delta]=\omega\big(\log(1/\delta)\big)\) as \(\delta\to0\). This concept can be extended to \(\zeta\)-asymptotic optimality. An algorithm is said to be {\em \(\zeta\)-asymptotically optimal} for \(\zeta > 1\) if:
\begin{align}
\limsup_{\delta \to 0} \frac{\EE_{\mu} [R(\tau_{\delta}) ]}{\log (1/\delta)} \leq \zeta \sum_{i\in [K] \setminus \{1\}}\frac{\Delta_i}{\kl(\mu_i,\mu_1)}.    \label{eqn:beta_opt}
\end{align}
The question is to explore the fundamental limits of \(\EE[\tau_\delta]\)  as a function of $\zeta$ as \(\delta\) vanishes under the constraint in \eqref{eqn:beta_opt}. Specifically, does \(\EE[\tau_\delta]\) scale as \(\Theta\big(\log(1/\delta)\big)\), and if it does, what is the exact constant involved? This remains an open question for future research.}

\section*{Acknowledgments}
We thank the anonymous reviewers for their helpful comments. T.~Jin  and V.~Y.~F.~Tan are supported by a Singapore Ministry of Education (MOE) AcRF Tier 2 grant under grant number A-8004062-00-00.

% \blue{This is a bit too sudden. We should recap the definition of  asymptotic optimality (which is $1$-asymptotic optimality). Recall from Theorem 5 that for asymptotic optimality, the stopping time/sample complexity satisfies $\mathbb{E}[\tau_\delta]= \omega( \log (1/\delta))$. We should then motivate that notion of $\beta$-asymptotic optimality as a relaxed notion of asymptotic optimality in which the expected cumulative regret differs from its fundamental limit by a multiplicative factor of $\beta$. A natural question is then the fundamental limit on $\mathbb{E}[\tau_\delta]$ and whether it is still $\omega( \log (1/\delta))$. }

% \red{Lastly, one interesting direction is exploring the trade-off between sample complexity and regret. We may consider the following notion of $\beta$-asymptotic optimality. We say that an algorithm is {\em $\beta$-optimal} if for some $\beta>0$,
% \begin{align*}
% \limsup_{\delta \to 0} \frac{\EE_{\mu}[\tau_{\delta}]}{\log (1/\delta)} \leq \beta \sum_{i\in [K]}\frac{\Delta_i}{\kl(\mu_i,\mu_1)}.
% \end{align*}
% Let
% \begin{align*}
%    \beta^* := \inf \bigg\{ x>0: \sum_{i\in [K]\setminus \{1\}} \frac{\Delta_i}{\kl(\mu_i,,\mu_1-x)} \leq \beta \sum_{i\in [K]}\frac{\Delta_i}{\kl(\mu_i,\mu_1)} \bigg\}.
% \end{align*}
% We conjecture that the following lower bound holds. Any algorithm that achieves $\beta$-asymptotic optimality must incur sample complexity
% \begin{align*}
% \sum_{i\in [K]} \frac{\log (1/\delta)}{\kl(\mu_i, \mu_1 - \beta^*)} + o(\log(1/\delta)).
% \end{align*}}

%\newpage
\vspace{20pt}
\appendix
\section{Auxiliary Results}
\label{appendix_auxiliary}
\subsection{Exponential Family Distributions}

For exponential families of distribution, the Kullback--Leibler (KL) divergence can be equivalently expressed as the Bregman divergence between the parameters or the means. Specifically, for two distributions $\nu_{\theta}$ and $\nu_{\theta'}$ with means $\mu$ and $\mu'$, it holds that 
$$
\begin{aligned}
\mathrm{kl}\left(\mu, \mu' \right)&=b\left(\theta'\right)-b(\theta)-b'(\theta)\left(\theta'-\theta\right) \\
&= b^*(\mu) -b^*(\mu') - \left(b'\right)^{-1} (\mu')(\mu-\mu')
\end{aligned}
$$
where $b^*$ is the convex conjugate of $b$.

\begin{lemma}[\citet{menard2017minimax}]
\label{lemma_klmean}
For any $\mu, \mu' \in I$, it holds that
$$
\mathrm{kl}\left(\mu, \mu' \right)\ge \frac {(\mu-\mu')^2} {2V}.
$$
\end{lemma}
%Equation~(1) of \citet{menard2017minimax}

\begin{lemma}[Maximal Inequality~\citep{menard2017minimax}]\label{lem:maximal-inequality}
Let $N_1$ and $N_2$ be two real numbers, and $\hat \mu_{n}$ be the empirical mean of $n$ i.i.d.\ random variables drawn according to some exponential family distribution with mean $\mu$. Let $V$ be the maximum variance of the distribution. Then, for every $x\leq \mu$,
\begin{align}\label{eq:kl-1}
\begin{split}
    \PP(\exists N_1\leq n \leq N_2, \hat \mu_n\leq x)&\leq e^{-N_1\cdot \kl(x,\mu)},  \\
    \PP(\exists N_1\leq n \leq N_2, \hat \mu_n\leq x)&\leq e^{-N_1(x-\mu)^2/(2V)}.
\end{split}
\end{align}
Moreover, for every $x\geq \mu$, 
\begin{align}\label{eq:kl-2}
     \PP(\exists N_1\leq n \leq N_2, \hat \mu_n\geq x)&\leq e^{-N_1\cdot \kl(x,\mu)}.
\end{align}
\end{lemma}

\subsection{Other Technical Lemmas}
\begin{lemma}
\label{lemma_convergent}
Consider any function $f: [0,\infty)^2 \to \mathbb R$. If $f$ is continuous  on $[0,\infty)^2$, then for any sequence $(x_n)\subset\mathbb{R}$ and convergent sequence $(y_n) \subset\mathbb{R}$, it holds that
$$
\liminf_{n\to \infty} f\left(x_n, y_n\right) \le f\left(\liminf_{n\to \infty} x_n, \lim_{n\to \infty}  y_n\right).
$$
\end{lemma}

%  and monotonically increasing in the first argument when the second is fixed
% \begin{proof}
% Since $f$ is continuous, it holds that
% \begin{align*}
% f\left(\liminf_{n\to \infty} x_n, \lim_{n\to \infty}  y_n\right) &= 
% f\left(\lim _{n \to \infty}\left(\inf _{m \geq n} x_m\right), \lim _{n \to \infty} y_{n}\right) \\
% &= \lim _{n \to \infty}   f\left(\inf _{m \geq n} x_m, y_{n}\right).
% \end{align*}

% Next, as $f$ is increasing in the first argument, we have
% \begin{align*}
% \lim _{n \to \infty}   f\left(\inf _{m \geq n} x_m, y_{n}\right) &=   \lim _{n \to \infty}   \inf _{m \geq n}  f\left(x_m, y_{n}\right) = \liminf_{n\to \infty} f\left(x_n, y_n\right).
% \end{align*}

% Therefore, we can conclude 
% $$
% \liminf_{n\to \infty} f\left(x_n, y_n\right) = f\left(\liminf_{n\to \infty} x_n, \lim_{n\to \infty}  y_n\right)
% $$
% as desired.
% \end{proof}

\begin{proof}
For the sake of brevity, we write $x_0:=\liminf_{n\to\infty}x_n$ and $y_0:=\lim_{n\to\infty}y_n$ (since $(y_n)$ is convergent). By a property of the   $\liminf$, there exists a convergent subsequence $(x_{n_k})$ of $(x_n) $ such that $\lim_{k\to\infty}x_{n_k}=x_0$. Along this subsequence, we also have $\lim_{k\to\infty}y_{n_k}=y_0$, since every subseqence of a convergent sequence converges to the same limit. By the definition of $\liminf$,
\begin{equation}
\liminf_{n\to\infty}f(x_n,y_n)=\inf E    \label{eqn:inf_def}
\end{equation}
where $E$ is the set of subsequential limits of the sequence $(f(x_n,y_n))$. Now, note that $\left(f(x_{n_k},y_{n_k})\right)$ is convergent. Indeed, its limit is 
$$
\lim_{k\to\infty} f(x_{n_k}, y_{n_k}) = f\left( \lim_{k\to\infty}x_{n_k}, \lim_{k\to\infty}y_{n_k} \right)=f(x_0,y_0)
$$
where the first equality follows from the continuity of $f$. As a result, we may upper bound in the infimum in \eqref{eqn:inf_def} by choosing the convergent subsequence $(f(x_{n_k},y_{n_k}))$, yielding
$$
\liminf_{n\to\infty}f(x_n,y_n)\le \lim_{k\to\infty}f(x_{n_k}, y_{n_k})=f(x_0,y_0) 
$$
as previously established. This completes the proof of Lemma~\ref{lemma_convergent}. 
\end{proof}

\section{Proof of Theorem~\ref{theorem_lowerbound}}
\label{appendix_proof_theorem_lowerbound}

\begin{proof}%[Proof of Theorem~\ref{theorem_lowerbound}]
For a fixed confidence level $\delta\in (0,1)$ and instance $\bmu \in \mathcal{M}$, consider any $\delta$-PAC best arm identification algorithm. Since the hardness parameter $\mathrm I^{*}(\bmu)$ is finite, the situation that ${\E_{\bmu}[ R(\tau_\delta)]}$ is infinite is trivial. Henceforth, we assume that ${\E_{\bmu}[ R(\tau_\delta)]}$ is finite. 

Recall that $\mathrm{Alt}(\bmu)$ represents the set of alternative instances such that arm $1$ is not the best arm. Consider an arbitrary instance $\blambda$ in $\mathrm{Alt}(\mu)$. By applying the \emph{transportation} inequality \citep[Lemma 1]{kaufmann2016complexity} and the KL divergence for the underlying exponential family, we have 
$$
\sum_{i=1}^{K} \mathbb{E}_{\bmu}[N_{i}(\tau_{\delta})]\cdot \mathrm{kl}\left(\mu_i, \lambda_i \right) \ge \mathrm{kl}_{\mathcal B}(\delta, 1-\delta).
$$

Since the above inequality holds for all instances in $\mathrm{Alt}(\bmu)$, we have 
\begin{align}
\mathrm{kl}_{\mathcal B}(\delta, 1-\delta) &\le \inf _{\blambda \in \mathrm{Alt}(\bmu)} \sum_{i=1}^{K} \mathbb{E}_{\bmu}[N_{i}(\tau_{\delta})]\cdot \mathrm{kl}\left(\mu_i, \lambda_i \right)  \label{equation_lowerbound1}
\\
&\le \inf _{\blambda \in \mathrm{Alt}(\bmu), \lambda_1 = \mu_1} \sum_{i=1}^{K} \mathbb{E}_{\bmu}[N_{i}(\tau_{\delta})] \cdot \mathrm{kl}\left(\mu_i, \lambda_i \right) \notag \\
&= \inf _{\blambda \in \mathrm{Alt}(\bmu), \lambda_1 = \mu_1} \sum_{i=2}^{K} \mathbb{E}_{\bmu}[N_{i}(\tau_{\delta})] \cdot \mathrm{kl}\left(\mu_i, \lambda_i \right)  \notag \\
& = \inf _{\blambda \in \mathrm{Alt}(\bmu), \lambda_1 = \mu_1}  {\E_{\bmu}[ R(\tau_\delta)]} \left(\sum_{i=2}^{K} \frac{\mathbb{E}_{\bmu}[N_{i}(\tau_{\delta})]}{\E_{\bmu}[ R(\tau_\delta)]} \cdot \mathrm{kl} \left(\mu_{i}, \lambda_{i}\right)\right) \notag \\*
&= \E_{\bmu}[ R(\tau_\delta)] \inf _{\blambda \in \mathrm{Alt}(\bmu), \lambda_1 = \mu_1} \left(\sum_{i=2}^{K} \frac{\Delta_i \mathbb{E}_{\bmu}[N_{i}(\tau_{\delta})]}{\E_{\bmu}[ R(\tau_\delta)]} \cdot \frac{\mathrm{kl} \left(\mu_{i}, \lambda_{i}\right)} {\Delta_i}\right) . \notag 
\end{align}

Due to the law of total expectation,  the regret ${\E_{\bmu}[ R(\tau_\delta)]}$ can be decomposed as follows:
$$
\E_{\bmu}[ R(\tau_\delta)]  =\E_{\bmu}\left[\sum_{t=1}^{\tau_\delta} \Delta_{A_t}\right] = \sum_{i=2}^K \Delta_i \mathbb{E}_{\bmu}[N_{i}(\tau_{\delta})].
$$

As a result, the vector $$\left(0, \frac{\Delta_2 \mathbb{E}_{\bmu}[N_{2}(\tau_{\delta})]}{\E_{\bmu}[ R(\tau_\delta)]}, \frac{\Delta_3 \mathbb{E}_{\bmu}[N_{3}(\tau_{\delta})]}{\E_{\bmu}[ R(\tau_\delta)]},\ldots,\frac{\Delta_K \mathbb{E}_{\bmu}[N_{K}(\tau_{\delta})]}{\E_{\bmu}[ R(\tau_\delta)]}\right) $$ constitutes a probability distribution in $\mathcal P _{K}$. 

Hence, we can derive the following:
\begin{equation}
\label{equation_lowerbound2}
\begin{aligned}
\mathrm{kl}_{\mathcal B}(\delta, 1-\delta) &\le \E_{\bmu}[ R(\tau_\delta)] \inf _{\blambda \in \mathrm{Alt}(\bmu), \lambda_1 = \mu_1} \left(\sum_{i=2}^{K} \frac{\Delta_i \mathbb{E}_{\bmu}[N_{i}(\tau_{\delta})]}{\E_{\bmu}[ R(\tau_\delta)]} \cdot \frac{\mathrm{kl} \left(\mu_{i}, \lambda_{i}\right)} {\Delta_i}\right) \\
&\le \E_{\bmu}[ R(\tau_\delta)] \sup _{w \in \mathcal P _{K}, w_1 = 0} \inf _{\blambda \in \mathrm{Alt}(\bmu), \lambda_1 = \mu_1} \left(\sum_{i=2}^{K}  \frac{w_i \mathrm{kl} \left(\mu_{i}, \lambda_{i}\right)} {\Delta_i}\right) .
\end{aligned}
\end{equation}

To establish the desired inequality $\E_{\bmu}[ R(\tau_\delta)] \geq \mathrm I^{*}(\bmu) \mathrm{kl}_{\mathcal B}(\delta, 1-\delta)$, it suffices to show 
$$
\sup _{w \in \mathcal P _{K}, w_1 = 0} \inf _{\blambda \in \mathrm{Alt}(\bmu), \lambda_1 = \mu_1} \left(\sum_{i=2}^{K}  \frac{w_i \mathrm{kl} \left(\mu_{i}, \lambda_{i}\right)} {\Delta_i}\right)  = \mathrm I^{*}(\bmu)^{-1}.
$$

By decomposing the feasible set (i.e., $\blambda \in \mathrm{Alt}(\bmu)$ with $ \lambda_1 = \mu_1$), we can get
\begin{align*}
    &\phantom{\ =\ } \inf _{\blambda \in \mathrm{Alt}(\bmu), \lambda_1 = \mu_1} \left(\sum_{i=2}^{K}  \frac{w_i \mathrm{kl} \left(\mu_{i}, \lambda_{i}\right)} {\Delta_i}\right)  \\
    &=  \min_{j > 1} \inf_{\blambda: \lambda_1 = \mu_1, \lambda_j > \lambda_1} \left(\sum_{i=2}^{K}  \frac{w_i \mathrm{kl} \left(\mu_{i}, \lambda_{i}\right)} {\Delta_i}\right)  \\
    &=  \min_{j > 1} \inf_{\blambda: \lambda_1 = \mu_1, \lambda_j > \lambda_1}\frac{w_j \mathrm{kl} \left(\mu_{j}, \lambda_{j}\right)} {\Delta_j} \\
    &=  \min_{j > 1}  \frac{w_j \mathrm{kl} \left(\mu_{j}, \mu_1 \right)} {\Delta_j}.
\end{align*}

Now consider the outer optimization problem:
$$
\sup _{w \in \mathcal P _{K}, w_1 = 0} \inf _{\blambda \in \mathrm{Alt}(\bmu), \lambda_1 = \mu_1} \left(\sum_{i=2}^{K}  \frac{w_i \mathrm{kl} \left(\mu_{i}, \lambda_{i}\right)} {\Delta_i}\right) = \sup _{w \in \mathcal P _{K}, w_1 = 0} \min_{j > 1}  \frac{w_j \mathrm{kl} \left(\mu_{j}, \mu_1 \right)} {\Delta_j}.
$$
Obviously, the supremum (maximum) is attained if and only if the values of $\frac{w_j \mathrm{kl} \left(\mu_{j}, \mu_1 \right)}{\Delta_j}$ are the same across all $j > 1$. Since $\{w_j\}_{j=1}^{K}$ forms a probability distribution,   this occurs when
$$
w_j = \frac{{\Delta_j} / {\mathrm{kl} \left(\mu_{j}, \mu_1 \right)}}{\sum_{i=2}^{K} {\Delta_i} / {\mathrm{kl} \left(\mu_{i}, \mu_1 \right)}}
$$
for all $j > 1$.

Therefore, we arrive at
\begin{align*}
     \sup _{w \in \mathcal P _{K}, w_1 = 0} \inf _{\blambda \in \mathrm{Alt}(\bmu), \lambda_1 = \mu_1} \left(\sum_{i=2}^{K}  \frac{w_i \mathrm{kl} \left(\mu_{i}, \lambda_{i}\right)} {\Delta_i}\right) 
    = \left(\sum_{i=2}^{K} \frac {\Delta_i} {\mathrm{kl} \left(\mu_{i}, \mu_1 \right)} \right)^{-1} 
    =  \mathrm I^{*}(\bmu)^{-1}.
\end{align*}
Finally, since $\lim_{\delta \rightarrow 0 }  \frac {\mathrm{kl}_{\mathcal B}(\delta, 1-\delta) }  {\log(1/\delta)} =1 $, letting $\delta \to 0 $ yields 
$$
\liminf_{\delta\to 0} \frac{\E_{\bmu}[ R(\tau_\delta)]} {\log (1/\delta)} \ge \mathrm I^{*}(\bmu).
$$
This completes the proof of Theorem~\ref{theorem_lowerbound}.
\end{proof}

\section{Proof of Theorem~\ref{theorem_impossible}}
\label{appendix_proof_theorem_impossible}

\begin{proof}%[Proof of Theorem~\ref{theorem_impossible}]
Consider any asymptotically optimal $\delta$-PAC algorithm and bandit instance $\bmu \in \mathcal{M}$. By combining the asymptotic lower bound in Theorem~\ref{theorem_lowerbound} with the definition of asymptotic optimality in Definition~\ref{definition_optimal}, we   obtain 
\begin{equation}
\lim_{\delta\to 0} \frac{\E_{\bmu}[ R(\tau_\delta)]} {\log (1/\delta)} = \mathrm I^{*}(\bmu).
    \label{eqn:lim_delta_ratio1}
\end{equation}

First, consider Inequality~\eqref{equation_lowerbound2} in the proof of Theorem~\ref{theorem_lowerbound}. By dividing this inequality by $\log(1/\delta)$ and allowing $\delta$ to approach zero,  we have
$$
1 \le \lim_{\delta\to 0}  \mathrm I^{*}(\bmu) \inf _{\blambda \in \mathrm{Alt}(\bmu), \lambda_1 = \mu_1} \left(\sum_{i=2}^{K} \frac{\Delta_i \mathbb{E}_{\bmu}[N_{i}(\tau_{\delta})]}{\E_{\bmu}[ R(\tau_\delta)]} \cdot \frac{\mathrm{kl} \left(\mu_{i}, \lambda_{i}\right)} {\Delta_i}\right)  \le 1,
$$
which leads to 
\begin{align*}
\lim_{\delta\to 0} \inf _{\blambda \in \mathrm{Alt}(\bmu), \lambda_1 = \mu_1} \left(\sum_{i=2}^{K} \frac{\Delta_i \mathbb{E}_{\bmu}[N_{i}(\tau_{\delta})]}{\E_{\bmu}[ R(\tau_\delta)]} \cdot \frac{\mathrm{kl} \left(\mu_{i}, \lambda_{i}\right)} {\Delta_i}\right)  = \mathrm I^{*}(\bmu) ^{-1} .
\end{align*}

Based on the analysis in the proof of Theorem~\ref{theorem_lowerbound}, the above equation is equivalent to 
\begin{align*}
\lim_{\delta\to 0} \min_{i > 1}  \frac{\Delta_i \mathbb{E}_{\bmu}[N_{i}(\tau_{\delta})]}{\E_{\bmu}[ R(\tau_\delta)]} \cdot  \frac{ \mathrm{kl} \left(\mu_{i}, \mu_1 \right)} {\Delta_i}= \mathrm I^{*}(\bmu) ^{-1},
\end{align*}
which holds if and only if 
\begin{equation}
\lim_{\delta\to 0} \frac{\mathbb{E}_{\bmu}[N_{i}(\tau_{\delta})]}{\E_{\bmu}[ R(\tau_\delta)]} = \frac 1 { \mathrm I^{*}(\bmu) \mathrm{kl} \left(\mu_{i}, \mu_1 \right)} \label{eqn:lim_delta_ratio2}
\end{equation}
for all $i>1$.

Consequently, by combining~\eqref{eqn:lim_delta_ratio1} and~\eqref{eqn:lim_delta_ratio2},  we establish that for all $i>1$,
\begin{equation}
\lim_{\delta\to 0} \frac{\mathbb{E}_{\bmu}[N_{i}(\tau_{\delta})]}{\log(1/\delta)} = \frac 1 {  \mathrm{kl} \left(\mu_{i}, \mu_1 \right)}.\label{eqn:N_i_conv}    
\end{equation}

Next, consider Inequality~\eqref{equation_lowerbound1}. Similarly,  dividing this inequality by $\log(1/\delta)$ and allowing $\delta$ to approach zero yields
\begin{equation}
\label{equation_contradiction}
   \lim_{\delta\to 0} \inf _{\blambda \in \mathrm{Alt}(\bmu)} \sum_{i=1}^{K} \frac{\mathbb{E}_{\bmu}[N_{i}(\tau_{\delta})]}{\log(1/\delta)} \cdot \mathrm{kl}\left(\mu_i, \lambda_i \right)  \ge 1.
\end{equation}

Note that
\begin{align*}
     &\phantom{\ = \ }\inf _{\blambda \in \mathrm{Alt}(\bmu)} \sum_{i=1}^{K} \frac{\mathbb{E}_{\bmu}[N_{i}(\tau_{\delta})]}{\log(1/\delta)} \cdot \mathrm{kl}\left(\mu_i, \lambda_i \right)  \\
     &= \min_{j > 1} \inf_{\blambda:  \lambda_j > \lambda_1} \sum_{i=1}^{K} \frac{\mathbb{E}_{\bmu}[N_{i}(\tau_{\delta})]}{\log(1/\delta)} \cdot \mathrm{kl}\left(\mu_i, \lambda_i \right) \\
     &= \min_{j > 1} \inf_{\blambda:  \lambda_j > \lambda_1}   \left ( \frac{\mathbb{E}_{\bmu}[N_{1}(\tau_{\delta})]}{\log(1/\delta)} \cdot \mathrm{kl}\left(\mu_1, \lambda_1 \right)  +  \frac{\mathbb{E}_{\bmu}[N_{j}(\tau_{\delta})]}{\log(1/\delta)} \cdot \mathrm{kl}\left(\mu_j, \lambda_j \right)  \right) \\
     &= \min_{j > 1} \ \inf_{\lambda\in I}   \left ( \frac{\mathbb{E}_{\bmu}[N_{1}(\tau_{\delta})]}{\log(1/\delta)} \cdot \mathrm{kl}\left(\mu_1, \lambda \right)  +  \frac{\mathbb{E}_{\bmu}[N_{j}(\tau_{\delta})]}{\log(1/\delta)} \cdot \mathrm{kl}\left(\mu_j, \lambda \right)  \right).
\end{align*}

Define the function $\Phi(x,y)$ for $x, y \ge 0$ as follows:
\begin{equation}
\label{equation_phi}
    \Phi(x,y) := \inf_{\lambda \in I } \left( x \cdot \mathrm{kl}\left(\mu_1, \lambda \right) + y \cdot \mathrm{kl}\left(\mu_j, \lambda \right) \right),
\end{equation}
which frequently appears in the analysis of best arm identification. The properties of the function $\Phi$ have been thoroughly explored; see \citet[Lemma 2]{russo2020simple} for an example. Specifically, the function $\Phi : [0,\infty)^2\to\mathbb{R}$ has the following properties:
\begin{itemize}
    \item $\Phi$ is continuous  on $[0,\infty)^2$.
    \item For $\mu_1 > \mu_j$, $\Phi$ is strictly increasing in the first (resp.\ second) argument when the second (resp.\ first) is fixed. 
    \item The infimum in $\Phi$ can be replaced with a minimum, and this minimum is attained when $\lambda = \frac{x}{x+y} \mu_1+\frac{y}{x+y} \mu_j$.
\end{itemize}

A direct consequence of the last property is that for any finite $y\ge 0 $,
$$
\lim_{x\to + \infty} \Phi(x, y) = y \cdot \mathrm{kl}\left(\mu_j, \mu_1 \right).
$$

In the following, we will show $\mathbb{E}_{\bmu}[N_{1}(\tau_{\delta})] = \omega({\log (1/\delta)})$ via contradiction. 
Assume, to the contrary, that there exists a finite constant $c > 0$ such that
$$
\liminf_{\delta\to 0} \frac{\mathbb{E}_{\bmu}[N_{1}(\tau_{\delta})]}{\log(1/\delta)} \le c .
$$

Then we have
\begin{align}
&\phantom{\ = \ }\liminf_{\delta\to 0} \inf _{\blambda \in \mathrm{Alt}(\bmu)} \sum_{i=1}^{K} \frac{\mathbb{E}_{\bmu}[N_{i}(\tau_{\delta})]}{\log(1/\delta)} \cdot \mathrm{kl}\left(\mu_i, \lambda_i \right)  \notag \\
&= \liminf_{\delta\to 0}  \ \min_{j > 1} \ \Phi\left(\frac{\mathbb{E}_{\bmu}[N_{1}(\tau_{\delta})]}{\log(1/\delta)}, \frac{\mathbb{E}_{\bmu}[N_{j}(\tau_{\delta})]}{\log(1/\delta)} \right) \notag \\
&= \min_{j > 1} \ \liminf_{\delta\to 0}  \  \Phi\left(\frac{\mathbb{E}_{\bmu}[N_{1}(\tau_{\delta})]}{\log(1/\delta)}, \frac{\mathbb{E}_{\bmu}[N_{j}(\tau_{\delta})]}{\log(1/\delta)} \right) \notag \\
&\le \min_{j > 1} \   \Phi\left(\liminf_{\delta\to 0} \frac{\mathbb{E}_{\bmu}[N_{1}(\tau_{\delta})]}{\log(1/\delta)}, \lim_{\delta\to 0} \frac{\mathbb{E}_{\bmu}[N_{j}(\tau_{\delta})]}{\log(1/\delta)} \right) \label{equation_impossible1} \\
&\le  \min_{j > 1} \   \Phi\left(c, \frac 1 {  \mathrm{kl} \left(\mu_{j}, \mu_1 \right)} \right) \notag \\
&< \min_{j > 1} \lim_{x\to +\infty}   \Phi\left(x, \frac 1 {  \mathrm{kl} \left(\mu_{j}, \mu_1 \right)} \right)  \notag\\
&=  \min_{j> 1} \   \frac {  \mathrm{kl} \left(\mu_{j}, \mu_1 \right)} {  \mathrm{kl} \left(\mu_{j}, \mu_1 \right)} \notag \\*
&= 1\notag
\end{align}
where Line~\eqref{equation_impossible1} is due to Lemma~\ref{lemma_convergent} and the above-stated properties of  the function $\Phi$. Note from Equation~\eqref{eqn:N_i_conv} that for each $j>1$, $\frac{\mathbb{E}_{\bmu}[N_{j}(\tau_{\delta})]}{\log(1/\delta)}$ converges as $\delta\to0$.

Therefore, we can conclude that
$$
\liminf_{\delta\to 0} \inf _{\blambda \in \mathrm{Alt}(\bmu)} \sum_{i=1}^{K} \frac{\mathbb{E}_{\bmu}[N_{i}(\tau_{\delta})]}{\log(1/\delta)} \cdot \mathrm{kl}\left(\mu_i, \lambda_i \right)  < 1,
$$
which contradicts  Inequality~\eqref{equation_contradiction}. 
Consequently, our claim that $\mathbb{E}_{\bmu}[N_{1}(\tau_{\delta})] = \omega({\log (1/\delta)})$ holds.

Finally, as $\mathbb{E}_{\bmu}[\tau_{\delta}]  = \sum_{i=1}^K \mathbb{E}_{\bmu}[N_{i}(\tau_{\delta})]\ge \mathbb{E}_{\bmu}[N_{1}(\tau_{\delta})]$, it holds that 
$$
\mathbb{E}_{\bmu}[\tau_{\delta}] = \omega({\log (1/\delta)})
$$
as desired.
\end{proof}

\section{Proof of Theorem~\ref{theorem_upper}}
\label{appendix_proof_theorem_upper}
The proof of Theorem~\ref{theorem_upper} consists of three parts. 
For ease of notation, we write $\gamma = \log \log(1/\delta)$ and omit the dependence of the bandit instance $\bmu$ in this appendix. Specifically, we abbreviate $\EE_{\bmu}$ and $\PP_{\bmu}$ as $\EE$ and $\PP$. Furthermore, for each arm $i\in[K]$, we define $\hat{\mu}_{is}$ as the empirical mean of arm $i$ based on its first $s$ pulls. {See Table~\ref{table:notation} for essential notations used throughout the proofs.}

%Definition: $\hat{\mu}_{is}$, $L_{is}^f$, $U_{is}^f$, $U_{is}^g$

\begin{table}[h]
\small
\begin{center}
\caption{{Frequently used notations in the proofs.}}
\label{table:notation}
\begin{tabular}{cp{12cm}}
\toprule
{Symbol} & {Definition} \\[0.5ex] 
\midrule
{$\gamma$} & {$\gamma=\log \log (1/\delta)$ } \\
{$\epsilon$} & {$\epsilon=\gamma^{-1/4}=\frac{1}{(\log \log (1/\delta))^{1/4}}$}  \\
{$\beta(\delta)$} & {$\beta(\delta)=1-\min\big\{\frac{1}{\log \log (1/\delta)}, \frac{1}{2}\big\}$} \\
{$h(\delta)$} & {$h(\delta)=\frac{\log (1/\delta)\cdot \gamma}{\epsilon}$} \\ 
{$t_r$} & {$t_r=2^{r}h(\delta)$ for $r\in \{0\}\cup \NN$} \\
%{$\cE(r)$} & {The event that the algorithm returns after $t_r$} \\
{$\mathrm{kl}(\mu,\mu')$}& {$\underline{\mathrm{kl}}(\mu,\mu')=\mathrm{kl}(\mu,\mu')\cdot \ind(\mu \leq \mu')$} \\
\bottomrule
\end{tabular}
\end{center}
\end{table}

\vspace{5pt}
% \begin{proof}[Proof of correctness]
\begin{proof}\textbf{of correctness$\quad$}
Consider any fixed confidence level $\delta \in (0, 1)$. If the algorithm returns a suboptimal arm, then there must exist some $t \in \mathbb{N}$ such that
\begin{align*}
    U^g_{1}(t) < \mu_1 \quad \text{or} \quad \exists i \geq 2, L^g_{i}(t) > \mu_i.
\end{align*}

For the best arm (i.e., arm $1$), according to Lemma~\ref{lem:maximal-inequality}, it holds that
\begin{align*}
&\phantom{\ =\ } \PP \left(\exists t \in \mathbb{N}, U^g_{1}(t) < \mu_1 \right) \\ 
&\le \PP\left(\exists s \in \mathbb{N}, \hat{\mu}_{1s}\leq \mu_1 \ \ \text{and} \ \kl(\hat{\mu}_{1s},\mu_1)>\frac{\log(2Ks^2/\delta)}{s} \right)\\
&\leq  \sum_{s=1}^{\infty}\PP \bigg(\hat{\mu}_{1s}\leq \mu_1 \ \ \text{and} \ \kl(\hat{\mu}_{1s},\mu_1)>\frac{\log(2Ks^2/\delta)}{s} \bigg)  \\
&\leq \sum_{s=1}^{\infty} \exp \left(-s\cdot \frac{\log(2Ks^2/\delta)}{s} \right) \\
&= \sum_{s=1}^{\infty} \frac{\delta}{2Ks^2}\\*
&\leq \frac{\delta}{K}.
\end{align*}

Similarly, for any suboptimal arm $i>1$, we can obtain
\begin{align*}
   \PP \left(\exists t \in \mathbb{N}, L^g_{i}(t) > \mu_i \right)  
   \leq \frac{\delta}{K}.
\end{align*}

Therefore, by applying union bound, we can bound the interested error probability as follows:
\begin{align*}
    \PP(i_{\mathrm{out}} \neq 1)&\leq \PP\bigg(\exists t \in \mathbb{N}: \{U^g_{1}(t) < \mu_1\}\lor \{\exists i \geq 2, L^g_{i}(t) > \mu_i\} \bigg) \notag \\
    &  \leq  \PP \left(\exists t \in \mathbb{N}, U^g_{1}(t) < \mu_1 \right) + \sum_{i=2}^K \PP \left(\exists t \in \mathbb{N}, L^g_{i}(t) > \mu_i \right)\\
    & \leq \delta.
\end{align*}
\end{proof}

% \begin{proof}[Proof of cumulative regret]
\begin{proof}\textbf{of cumulative regret$\quad$}
Let $\epsilon > 0$ and 
\begin{align*}
    h(\delta)=\frac{\log (1/\delta)\cdot \gamma}{\epsilon}.
\end{align*}

For the asymptotic analysis, we can pick $$\epsilon = \gamma^{-1/4}  = \frac 1  {(\log \log (1/\delta))^{1/4}}.$$ For $r \in \{0\}\cup \mathbb N$, we set  
$t_r=2^{r}h(\delta)$
and define the event
\begin{align*}
    \cE(r)&=\left\{\text{Algorithm \ref{algo1} returns after $t_{r}$} \right\}.
\end{align*}

Therefore,
\begin{align}
\label{eq:def-t0}
    t_0 =h(\delta)=\frac{\log (1/\delta)\cdot \gamma}{\epsilon}
\end{align}

Using the linearity of expectation, we can decompose the expected cumulative regret as follows:
\begin{align*}
    \E[ R(\tau_\delta)] &= \E[ R(t_0)] + \sum_{r=0}^{\infty} \E\big[ (R(t_{r+1})-R(t_{r}))\ind \left[ \cE(r) \right] \big] \\
    &\le \E[ R(t_0)] + \sum_{r=0}^{\infty} t_r \Delta_{\max} \PP(\cE(r)),
\end{align*}
where we denote $\Delta_{\max} = \max_{i\in [K]} \Delta_i$.

We introduce one shorthand notation:
$$
   \Pi^* =  \sum_{r=0}^{\infty} t_r \Delta_{\max} \PP(\cE(r)).
$$

In the following steps, we will bound $\E[ R(t_0)]$ and $\Pi^*$ separately. Specifically, we will show $\E[ R(t_0)] \le  \mathrm I^{*}(\bmu)\log (1/\delta) + o(\log (1/\delta))$ and $\Pi^* =  o(\log (1/\delta))$.

\paragraph{Bounding $\E[ R(t_0)]$.}
For any $\mu, \mu' \in I$, define $\underline{\mathrm{kl}}(\mu,\mu')=\mathrm{kl}(\mu,\mu')\cdot \ind(\mu \leq \mu')$. 
We introduce the following lemmas, which are generalizations of  \citep[Lemmas 10.7 and 10.8]{lattimore2020bandit} from Bernoulli distributions to distributions in the exponential family.

\begin{lemma}
\label{lem:Rt0-1}
Let $X_{1},X_2, \cdots, X_{t_0}$ be  i.i.d.\ random variables from a one-parameter exponential family  with mean $\mu_1$. Further, let $\epsilon>0$, and define
\begin{align*}
    {\psi}=\min\bigg\{t: \max_{s\in [t_0]}\ \underline{\mathrm{kl}}(\hat{\mu}_{1s},\mu_1-\epsilon)- \frac{f(t)}{s}\le 0 \bigg\}.
\end{align*}
Then, 
\begin{align*}
    \EE[{\psi}]\leq  1+ \frac {3V} {\epsilon^2}.
\end{align*}

\end{lemma}
\begin{lemma}
\label{lem:Rt0-2}
    Let $X_{1},X_2, \cdots, X_{t_0}$ be i.i.d. exponential family random variables with mean $\mu_i$. Further, let $\epsilon>0$, and define 
    \begin{align*}
        \kappa_i=\sum_{s=1}^{t_0}\ind \left\{\mathrm{kl}(\hat{\mu}_{is},\mu_1-\epsilon)\leq \frac{f(t_0)}{s} \right\}.
    \end{align*}
    Then, 
\begin{align*}
    \EE[\kappa_i]\leq \inf_{\epsilon' \in (0,\Delta_i-\epsilon)} \bigg(\frac{f(t_0)}{\mathrm{kl}(\mu_i+\epsilon',\mu_1-\epsilon)}+\frac{2V}{\epsilon'^2} \bigg).
\end{align*}    
\end{lemma}

Since $f(t)=3\log t$, it is straightforward to verify $\EE[\psi] =  o(\log (1/\delta))$ and $\EE[\kappa_i]= o(\log (1/\delta))$.

According to the definitions of $\psi$ and $\kappa_i$, we can bound $\E[ R(t_0)]$ as follows:
\begin{align*}
   \E[ R(t_0)]& =\sum_{i>1}\sum_{t=1}^{t_0} \Delta_i \EE \left[\ind\{A_{t}=i\}\right] \notag \\
    & \leq \EE[\psi]\Delta_{\max}+\sum_{i>1}\Delta_i\EE\bigg[\sum_{t=\psi+1}^{t_0} \ind \{A_{t}=i\}\bigg]\notag \\
    & \leq \EE[\psi]\Delta_{\max}+\sum_{i>1}\Delta_i\EE\bigg[\sum_{t=\psi+1}^{t_0} \ind \bigg\{A_{t}=i, \mathrm{kl}( \hat{\mu}_{i}(t-1), \mu_1-\epsilon)\leq \frac{f(t_0)}{N_{i}(t-1)}\bigg\}\bigg] \notag \\
    & \qquad \quad + \sum_{i>1}\Delta_i \EE\bigg[\sum_{t=\psi+1}^{t_0} \ind \bigg\{A_{t}=i, {\mathrm{kl}}(\hat{\mu}_{i}(t-1), \mu_1-\epsilon)> \frac{f(t_0)}{N_{i}(t-1)}\bigg\}\bigg] \notag \\
    & \leq  \EE[\psi]\Delta_{\max}+\sum_{i>1}\Delta_i\EE[\kappa_i] \nonumber\\*
    &\qquad\quad+{\sum_{i>1}\Delta_i \EE\bigg[\sum_{t=\psi+1}^{t_0} \ind \bigg\{A_{t}=i, {\mathrm{kl}}(\hat{\mu}_{i}(t-1), \mu_1-\epsilon)> \frac{f(t_0)}{N_{i}(t-1)}\bigg\}\bigg]}.
\end{align*}

Let 
$$
\Pi_1 = \sum_{i>1}\Delta_i\sum_{t=\psi+1}^{t_0} \ind \bigg\{A_{t}=i, {\mathrm{kl}}(\hat{\mu}_{i}(t-1), \mu_1-\epsilon)> \frac{f(t_0)}{N_{i}(t-1)}\bigg\}
$$
Then we have 
\begin{align}
\label{equation_jjj1}
   \E[ R(t_0)] \le \EE[\psi]\Delta_{\max}+\sum_{i>1}\Delta_i\EE[\kappa_i]+ \E[\Pi_1].
\end{align}

Consider
$$\cE_{\rm{exp}}(t)=\bigcup_{i\in [K]\setminus \{1\}}\left\{A_{t}=i, \mathrm{kl}(\hat{\mu}_{i}(t-1), \mu_1-\epsilon)> \frac{f(t_0)}{N_{i}(t-1)}\right\}.$$

Note that for any $t>\psi$, $U^f_{1}(t)\geq \mu_1-\epsilon$. Furthermore, if for some $i\in [K]\setminus \{1\}$, the condition $\left\{A_{t}=i, \mathrm{kl}(\hat{\mu}_{i}(t-1), \mu_1-\epsilon)> \frac{f(t_0)}{N_{i}(t-1)}\right\}$ holds, then $U^f_{i}(t)< \mu_1-\epsilon$, implying $A_{t}^f \neq i$. 
Therefore,  $\cE_{\rm{exp}}(t)$ is true only if $A_{t}=A_{t}^g$. 

The following lemma demonstrates that for 
$$T_{0}=\min\bigg\{\psi+\sum_{i>1}\sum_{t=\psi+1}^{t_0} \ind \bigg\{A_{t}=i, {\mathrm{kl}}(\hat{\mu}_{i}(t-1), \mu_1-\epsilon)\leq  \frac{f(t_0)}{N_{i}(t-1)}\bigg\}+ {\gamma}{\epsilon} \log\frac{1}{\delta} ,t_0\bigg\}$$ 
and $t\in (\psi, T_0]$, with high probability, the number of times that $\cE_{\rm{exp}}(t)$ occurs is less than $2 \epsilon \log({1}/{\delta})$.

\begin{lemma}
\label{lem:proexp}
For sufficiently small $\delta$, it holds that
    \begin{align*}
   \PP\bigg( \sum_{t=\psi+1}^{T_0}\ind \left\{ \cE_{\rm{exp}}(t)\right\} \geq {{\frac{2\log({1}/{\delta})}{(\log \log (1/\delta))^{1/4}}} }\bigg)\leq {\frac{1}{\log (1/\delta)\cdot(\log \log (1/\delta))^{5/4}}}.
\end{align*}
\end{lemma}

Note that
\begin{align*}
\E[\Pi_1] = \EE\big[\Pi_1\cdot\ind\{T_0=t_0\}\big] +  \EE\big[\Pi_1\cdot\ind\{T_0<t_0\}\big]. 
\end{align*}

{Recall  that $t_0=\log(1/\delta)\gamma/\epsilon$, $\gamma=\log \log (1/\delta)$, and $\epsilon=\gamma^{-1/4}$.} For the  case that $T_0=t_0$, by Lemma \ref{lem:proexp},
\begin{align*}
    \EE\big[\Pi_1\cdot\ind\{T_0=t_0\}\big] \leq \frac{1}{t_0}\cdot \Delta_{\max}\cdot t_0+\bigg(1-\frac{1}{t_0}\bigg)\Delta_{\max} \cdot 2 \epsilon \log({1}/{\delta}) =o\bigg(\log \frac{1}{\delta}\bigg).
\end{align*}

Now, consider the case that $T_0<t_0$. In this case, we have 
$$T_{0}=\psi+\sum_{i>1}\sum_{t=\psi+1}^{t_0} \ind \bigg\{A_{t}=i, {\mathrm{kl}}(\hat{\mu}_{i}(t-1), \mu_1-\epsilon)\leq  \frac{f(t_0)}{N_{i}(t-1)}\bigg\}+{\gamma}{\epsilon}  \log\frac{1}{\delta}.$$ 
By Lemma \ref{lem:proexp}, with probability $1-1/t_0$,
\begin{align*}
    \bigg(N_{1}(T_0)-N_{1}(\psi) \bigg)+{2}{\epsilon}  \log\frac{1}{\delta} >{\gamma}{\epsilon}  \log\frac{1}{\delta}.
\end{align*}
Therefore, for sufficiently small $\delta$,
\begin{align*}
    \PP\left(N_{1}(T_0)> \frac 1 2 {\gamma}{\epsilon}  \log \frac{1}{\delta}\right)\ge 1-\frac{1}{t_0}. 
\end{align*}
Then, from Lemma \ref{lem:maximal-inequality}, for sufficiently small $\delta$,
\begin{align*}
    {\PP\left(\forall s>\frac 1 2 {\gamma}{\epsilon}  \log \frac{1}{\delta}, \hat{\mu}_{1s}\geq \mu_1-\frac {\epsilon} 2 \right) \geq 1-\exp\bigg(-\frac{\log \delta^{-1}\cdot \gamma \epsilon^3}{16V} \bigg)}\geq 1-\frac{1}{t_0}. 
\end{align*}
Moreover, for sufficiently small $\delta$ and $t\in [t_0]$, if $\hat{\mu}_{1}(t-1)\geq \mu_1-\frac {\epsilon} 2$ and $N_1(t-1) \geq \frac 1 2 {\gamma}{\epsilon}  \log \frac{1}{\delta}$, then 
\begin{align*}
\label{eq:L1s}
    &L_{1}^g(t)   \geq \mu_1-\epsilon \notag \\
    \Leftarrow\quad   &{\mathrm{kl}\left(\hat{\mu}_1(t-1), \mu_1-\epsilon\right)}  > \frac{g(\delta,t_0)}{N_1(t-1)}\notag \\     
     \Leftarrow \quad  &{{\mathrm{kl}\left(\mu_1-\epsilon/2, \mu_1-\epsilon\right)}}   >\frac{g(\delta,t_0)}{N_1(t-1)}  \\ %\tag{due to Lemma~\ref{lemma_klmean}} \\
          \Leftarrow\quad   &{\frac{\epsilon^2} {8V}} > \frac{\log (2Kt_0^2/\delta)}{\gamma\epsilon/(2)  \cdot \log(1/\delta)}.  \tag{which is true for sufficiently small $\delta$}
\end{align*}
%\begin{align}
%\label{eq:L1s}
  %  \PP\left(\forall s>\frac 1 2 \log \frac{1}{\delta}, L^g_{1s}\geq \mu_1-\epsilon\right)\geq 1-\exp\bigg(-\frac{\log \delta^{-1}\cdot \epsilon^2}{4V} \bigg)\ge 1-\frac{1}{t_0}.
%\end{align}
Hence, we have 
\begin{align}
\label{eq:L1s1}
\PP\left(\forall t\in [ T_0,t_0], L^g_{1}(t)\geq \mu_1-\epsilon\right) \ge 1-\frac{2}{t_0} .
\end{align}
The following lemma shows that each suboptimal arm is pulled within the optimal range with high probability. 
\begin{lemma}
\label{lem:Rt0-4}
Assuming that $L^g_{1}(t)\geq \mu_1-\epsilon$ holds for all $t\geq T_0$,  for sufficiently small $\delta$, 
\begin{align*}
    &\PP\Bigg(\sum_{t = T_0+1}^{t_0} \ind\left\{A_{t}=i, \mathrm{kl}(\hat{\mu}_{i}(t-1), \mu_1-\epsilon)> \frac{f(t_0)}{N_{i}(t-1)}\right\} \nonumber\\*
    &\qquad\qquad\qquad>\max\bigg\{\frac{g(\delta, t_0) }{\mathrm{kl}(\mu_i+\epsilon,\mu_1-\epsilon)},\frac{2V\log (t_0)}{\epsilon^2}\bigg\}\Bigg)\leq \frac{1}{t_0}.
\end{align*}
\end{lemma}

Therefore, we have
\begin{align*}
     &\phantom{\ = \ }\EE\big[\Pi_1\cdot\ind\{T_0<t_0\}\big] \\
     &=\sum_{i>1}\Delta_i \EE\bigg[\sum_{t=\psi+1}^{t_0} \ind \bigg\{T_0<t_0, A_{t}=i, {\mathrm{kl}}(\hat{\mu}_{i}(t-1), \mu_1-\epsilon)> \frac{f(t_0)}{N_{i}(t-1)}\bigg\}\bigg] \notag \\
         &\leq \sum_{i>1}\Delta_i \EE\bigg[\sum_{t=\psi+1}^{T_0} \ind \bigg\{A_{t}=i, {\mathrm{kl}}(\hat{\mu}_{i}(t-1), \mu_1-\epsilon)> \frac{f(t_0)}{N_{i}(t-1)}\bigg\}\bigg] \notag \\ 
         & \quad +\sum_{i>1}\Delta_i \EE\bigg[\sum_{t=T_0+1}^{t_0} \ind \bigg\{A_{t}=i, {\mathrm{kl}}(\hat{\mu}_{i}(t-1), \mu_1-\epsilon)> \frac{f(t_0)}{N_{i}(t-1)}, \forall t\in [ T_0,t_0], L_{1}^{g}(t)\geq \mu_1-\epsilon \bigg\}\bigg] \notag \\
         & \quad +t_0\Delta_{\max} \PP\bigg(\exists t\in[T_0,t_0]: L_{1}^g(t)<\mu_1-\epsilon \bigg) \notag \\
         & \leq \Delta_{\max} {2\epsilon \log({1}/{\delta})}+t_0\Delta_{\max}\cdot   \PP\bigg( \sum_{t=\psi+1}^{T_0}\ind \left\{ \cE_{\rm{exp}}(t)\right\} \geq {2\epsilon \log({1}/{\delta})}\bigg) \tag{due to Lemma~\ref{lem:proexp}}\\
         &\quad  +\sum_{i>1}\Delta_{i}  \cdot t_0 \cdot \frac 1 {t_0} + \sum_{i>1} \Delta_{i} \max\bigg\{\frac{g(\delta, t_0) }{\mathrm{kl}(\mu_i+\epsilon,\mu_1-\epsilon)},\frac{2V\log (t_0)}{\epsilon^2}\bigg\}  \tag{due to Lemma~\ref{lem:Rt0-4}} \\
         & \quad +t_0\cdot \Delta_{\max} \cdot \frac{2}{t_0}  \tag{due to Inequality~\eqref{eq:L1s1}} \\
         &\leq \Delta_{\max} {2\epsilon \log({1}/{\delta})}+ 3\Delta_{\max} + \sum_{i>1}\Delta_{i} + \sum_{i>1} \Delta_{i} \max\bigg\{\frac{g(\delta, t_0) }{{\mathrm{kl}(\mu_i+\epsilon,\mu_1-\epsilon)}},\frac{2V\log (t_0)}{\epsilon^2}\bigg\}  \\
         & \leq o\bigg(\log \frac{1}{\delta}\bigg)+ \sum_{i>1}\frac{\Delta_{i}\log (1/\delta) }{\mathrm{kl}(\mu_i,\mu_1)},
\end{align*}
where in the last inequality, we use the fact that $g(\delta, t) = \log(2Kt^2/\delta)$.

Altogether, we can obtain
$$
\E[\Pi_1] \le  \sum_{i>1}\frac{\Delta_{i}\log (1/\delta) }{\mathrm{kl}(\mu_i,\mu_1)} + o\bigg(\log \frac{1}{\delta}\bigg).
$$

Combining with \eqref{equation_jjj1}, we   arrive at
\begin{align*}
\E[ R(t_0)] &\le    \sum_{i>1}\frac{\Delta_{i}\log (1/\delta) }{\mathrm{kl}(\mu_i,\mu_1)} + o\bigg(\log \frac{1}{\delta}\bigg) \\
&= \mathrm I^{*}(\bmu)\log (1/\delta) + o(\log (1/\delta))
\end{align*}
as desired.

\paragraph{Bounding $\Pi^*$.} To bound $\Pi^*$, it suffices to show the following lemma. The proof of Lemma~\ref{lem:I1-5} requires introducing the events $\cE_{0}(r)$, $\cE_{1}(r)$, $\cE_{2}(r)$ and $\cE_{3}(r)$. We refer to Appendix~\ref{subsection_important_lemma} for details.

\begin{lemma}
\label{lem:I1-5}
For sufficiently small $\delta$, it holds that
\begin{align*}
      \PP(\cE(r) ) \leq \frac{14}{t_r^2}.
\end{align*}
\end{lemma}

Then we have 
\begin{align*}
    \Pi^* &=  \sum_{r=0}^{\infty} t_r \Delta_{\max} \PP(\cE(r)) \\
    &\leq    \sum_{r=0}^{\infty} \frac{14\Delta_{\max}}{t_{r}} \\
    &\le \frac{14\Delta_{\max}}{h(\delta)} \\
    &= o(\log (1/\delta)).
\end{align*}

The proof of expected cumulative regret is complete.
\end{proof}

\vspace{10pt}
% \begin{proof}[Proof of sample complexity]
\begin{proof}\textbf{of sample complexity$\quad$}
Similar to the proof of expected cumulative regret, we can decompose the sample complexity as follows:
\begin{align*}
     \EE[\tau_{\delta}] & = t_0+\sum_{r=0}^{\infty} {\PP\left(\cE(r)\right)(t_{r+1}-t_{r} )}.
\end{align*}

Using Lemma~\ref{lem:I1-5}, we have 
\begin{align*}
    \sum_{r=0}^{\infty} {\PP\left(\cE(r)\right)(t_{r+1}-t_{r} )} \le \sum_{r=0}^{\infty} \frac {14}{t_r^2} \cdot {t_r} = \sum_{r=0}^{\infty} \frac {14}{2^{r}h(\delta)} =  o\bigg(\log \frac{1}{\delta}\bigg).
\end{align*}

Therefore, we can conclude that 
\begin{align*}
   \lim_{\delta\rightarrow 0} \frac{\EE[\tau_{\delta}]}{\log (1/\delta)\cdot (\log \log(1/\delta))^2 } = \lim_{\delta\rightarrow 0} \frac{t_0}{\log (1/\delta)\cdot (\log \log(1/\delta))^2 }=0.
\end{align*}

\end{proof}

\section{Proofs of Supporting Lemmas for Theorem~\ref{theorem_upper}}

% \begin{proof}[Proof of Lemma \ref{lem:Rt0-1}]
\begin{proof}\textbf{of Lemma \ref{lem:Rt0-1}$\quad$}
{Recall $V$ is the upper bound of the variances of the exponential family.} 
According to Lemma~\ref{lem:maximal-inequality},  we have
\begin{align*}
    \PP(\psi>t) &\leq \PP\bigg(\exists 1\leq s \leq t_0: \underline{\rm{kl}}(\hat{\mu}_{1s},\mu_1-\epsilon)>\frac{f(t)}{s} \bigg) \notag \\
            & \leq \sum_{s=1}^{t_0}\PP\bigg(\underline{\rm{kl}}(\hat{\mu}_{1s},\mu_1-\epsilon)>\frac{f(t)}{s} \bigg) \notag \\
            & =\sum_{s=1}^{t_0}\PP\bigg({\rm{kl}}(\hat{\mu}_{1s},\mu_1-\epsilon)>\frac{f(t)}{s},\hat{\mu}_{1s}<\mu_1-\epsilon\bigg) \notag \\
            &\le \sum_{s=1}^{t_0}\PP\bigg({\rm{kl}}(\hat{\mu}_{1s},\mu_1)>\frac{f(t)}{s} +  {\frac{\epsilon^2}{2V}},\hat{\mu}_{1s}<\mu_1\bigg) \notag \\
            & \leq \sum_{s=1}^{t_0}\exp\bigg(-s\bigg(\frac{\epsilon^2}{2V}+\frac{f(t)}{s} \bigg) \bigg) \notag \\
            & \leq \frac{1}{\exp(f(t))}\sum_{s=1}^{t_0}\exp (-\frac {s\epsilon^2} {2V}) \\ 
            &\le \frac{2V}{t^3 \epsilon^2}.
\end{align*}
Therefore, we can obtain
\begin{align*}
    \EE[\psi]= 1 + \sum_{t=1}^{\infty}
    \PP(\psi>t) \leq 1+ \frac {3V} {\epsilon^2}.
\end{align*}
\end{proof}

% \begin{proof} [Proof of Lemma \ref{lem:Rt0-2}]
\begin{proof}\textbf{of Lemma \ref{lem:Rt0-2}$\quad$}
Let $\epsilon' \in (0,\Delta_i-\epsilon)$ and $u=\frac {f(t_0)} {\kl(\mu_i+\epsilon',\mu_1-\epsilon)}$. Then
\begin{align*}
\EE[\kappa_i]&=\sum_{s=1}^{t_0}\PP\left(\kl(\hat{\mu}_{is},\mu_1-\epsilon)\leq \frac{f(t_0)}{s}\right) \notag \\
    & \leq \sum_{s=1}^{t_0} \PP\left(\hat{\mu}_{is}\geq \mu_i+\epsilon' \ \text{ or } \ \kl(\mu_i+\epsilon', \mu_1-\epsilon)\leq \frac{f(t_0)}{s}\right) \notag \\
  %  & \leq u+\sum_{s=\lceil u \rceil}^{t_0} \PP(\hat{\mu}_{1s}\geq \mu_1-\epsilon) \notag \\
    & \leq u+\sum_{s=\lceil u \rceil}^{t_0} \PP(\hat{\mu}_{is}\geq \mu_i+\epsilon')  \\
    & \leq u+\sum_{s=1}^{\infty} \exp \bigg(-s\cdot \kl(\mu_i+\epsilon',\mu_i) \bigg) \tag{due to Lemma \ref{lem:maximal-inequality}} \\
    & \leq \frac{f(t_0)}{\kl(\mu_i+\epsilon',\mu_1-\epsilon)}+\frac{1}{\kl(\mu_i+\epsilon',\mu_i)} \notag \\
    & \leq  \frac{f(t_0)}{\kl(\mu_i+\epsilon',\mu_1-\epsilon)}+{\frac{2V}{\epsilon'^2}}. \tag{due to Lemma \ref{lemma_klmean}}
\end{align*}
\end{proof}

% \begin{proof}[Proof of Lemma \ref{lem:proexp}]
\begin{proof}\textbf{of Lemma \ref{lem:proexp}$\quad$}
{Note that the claim of Lemma \ref{lem:proexp} is equivalent to showing that
\begin{align*}
    \PP\bigg( \sum_{t=\psi+1}^{T_0}\ind \left\{ \cE_{\rm{exp}}(t)\right\} \geq {2 \log({1}/{\delta})}\bigg)\leq \frac{1}{t_0}.
\end{align*}}
We note that for $t>\psi$, one of the following three disjoint events occurs:
\begin{itemize}
    \item $A_t=1$;
    \item for some $i>1$, $A_{t}=i$ and $ {\mathrm{kl}}(\hat{\mu}_{i}(t-1), \mu_1-\epsilon)> \frac{f(t_0)}{N_{i}(t-1)}$;
    \item for some $i>1$, $A_{t}=i$ and ${{\mathrm{kl}}(\hat{\mu}_{i}(t-1), \mu_1-\epsilon)\leq \frac{f(t_0)}{N_{i}(t-1)}}$.
\end{itemize}

Besides, for $t>\psi$, the event $\cE_{\exp}(t)$ only happens when the coin toss yields tails,  and the condition 
$${\mathrm{kl}}(\hat{\mu}_{i}(t-1), \mu_1-\epsilon)>\frac{f(t_0)}{N_{i}(t-1)}$$
holds true. 
Therefore,  by the definition of  $\cE_{\exp}(t)$, $\sum_{t = \psi + 1}^{T_0} \ind \left\{ \cE_{\exp}(t)\right \} $ is at most the number of heads that we toss ${\gamma}{\epsilon} \log\frac{1}{\delta} $ coins with bias $1-\beta(\delta)\leq 1/\gamma$.
From Hoeffding's bound, we have that for sufficiently small $\delta$, 
\begin{align*}
    \PP\bigg(\sum_{t = \psi + 1}^{T_0} \ind \left\{ \cE_{\exp}(t)\right \})-{\gamma}{\epsilon} \log\frac{1}{\delta} \cdot \frac{1}{\gamma}\geq  {\epsilon}   \log \frac{1}{\delta}  \bigg)&\leq \exp\bigg(-\frac{2\epsilon^2  \big(\log \frac{1}{\delta}\big)^2 }{{\gamma}{\epsilon} \log\frac{1}{\delta}} \bigg)\notag \\
    &\le \exp\bigg(-\frac{2\epsilon \log \frac{1}{\delta}}{\gamma} \bigg)\\ 
    &\leq \frac{1}{t_0},
\end{align*}
which completes the proof of Lemma~\ref{lem:proexp}.
\end{proof}

% \begin{proof}[Proof of Lemma \ref{lem:Rt0-4}]
\begin{proof}\textbf{of Lemma \ref{lem:Rt0-4}$\quad$}
Note that since $T_{0}>\psi$, $U^f_{1}(t)> \mu_1-\epsilon$ for $t>T_0$. Therefore, the event $\left\{A_{t}=i, \mathrm{kl}(\hat{\mu}_{i}(t-1), \mu_1-\epsilon)> \frac{f(t_0)}{N_{i}(t-1)}\right\}$ occurs only when $A_{t}=A_{t}^g=i$. Besides, the Algorithm returns when
\begin{align*}
    L^g_{1}(t)>\mu_1-\epsilon \ \text{ and } \ \forall i\in [K]\setminus \{1\}, U^g_{i}(t)\leq \mu_1-\epsilon.
\end{align*}
After pulling arm $i$ for $\max\bigg\{\frac{g(\delta, t_0) }{\mathrm{kl}(\mu_i+\epsilon,\mu_1-\epsilon)},\frac{2V\log (t_0)}{\epsilon^2}\bigg\}$ times, by Lemma \ref{lem:maximal-inequality},
\begin{align*}
    \PP( \hat{\mu}_{i}(t-1)\leq \mu_i+\epsilon) \geq 1-\frac{1}{t_0}.
\end{align*}
Furthermore, conditioned on the event that $\hat{\mu}_{i}(t-1)\leq \mu_i+\epsilon$ and arm $i$ is pulled for more than  $\max\bigg\{\frac{g(\delta, t_0) }{\mathrm{kl}(\mu_i+\epsilon,\mu_1-\epsilon)},\frac{2V\log (t_0)}{\epsilon^2}\bigg\}$ times, we have 
\begin{align*}
    U^g_{i}(t) & \leq \sup \bigg\{\mu\in I: \kl(\hat{\mu}_{i}(t-1),\mu)\leq \frac{g(\delta,t_0)}{N_{i}(t-1)}\bigg\} \notag \\
               & <  \sup \bigg\{\mu\in I: \kl(\hat{\mu}_{i}(t-1),\mu)\leq \mathrm{kl}(\mu_i+\epsilon,\mu_1-\epsilon) \bigg\} \notag \\
               & \le \mu_1-\epsilon.
\end{align*}

Recall that $g(\delta, t)=\log(2Kt^2/\delta)$ and $f(t)=3\log t$. For sufficiently small $\delta$, $g(\delta,t)>f(t)$ for all $t\in [t_0]$. Therefore, for all $t\in [t_0]$, $U^f_{i}(t)<U^g_{i}(t)<\mu_1-\epsilon$.

According to Algorithm \ref{algo1}, if it happens that $A_{t}^g=i\neq 1$, then
\begin{enumerate}
    \item $U^f_{1}(t)>L^g_{1}(t)>\mu_1-\epsilon>U^g_{i}(t)>\max_{j\in [K]\setminus\{1,i\}} U^g_{j}(t)\geq \max_{j\in [K]\setminus\{1,i\}} U^f_{j}(t)$, which implies that $A_{t}^f=1$.
    \item $L^g_{1}(t)=L^g_{A_{t}^f}(t)\geq \mu_1-\epsilon\geq \max_{j\in [K]\setminus \{1\}}U^g_{j}(t)$, which means that the algorithm returns.
\end{enumerate}

Therefore, it cannot be the case that $A_{t}^g=i\neq 1$, and we   obtain:
\begin{align*}
    &\PP\Bigg(\sum_{t = T_0+1}^{t_0} \ind\bigg\{A_{t}=i, \mathrm{kl}(\hat{\mu}_{i}(t-1), \mu_1-\epsilon)>\frac{f(t_0)}{N_{i}(t-1)}\bigg\}  \nonumber\\*
    &\qquad\qquad\qquad>\max\bigg\{\frac{g(\delta, t_0) }{\mathrm{kl}(\mu_i+\epsilon,\mu_1-\epsilon)},\frac{2V\log (t_0)}{\epsilon^2}\bigg\}\Bigg)\leq \frac{1}{t_0}.
\end{align*}
\end{proof}

\subsection{Proof of Lemma~\ref{lem:I1-5}}
\label{subsection_important_lemma}
Let $L_{r}$ be the number of heads in coin tosses before $t_r/2$, and 
\begin{align*}
    M(r)=\frac{2V\log (Kt_{r}^2) }{\epsilon^2}+ \frac{\max\{ f(t_r), g(\delta, t_r)\} }{\min _{i>1} \mathrm{kl}(\mu_i+\epsilon,\mu_1-\epsilon)}.
\end{align*}

We define the following events:
\begin{align*}
    \cE_{0}(r)&=\bigg\{L_{r}> \frac{t_r}{8}\bigg\}, \notag \\
    \cE_{1}(r)& =\bigg\{\forall t\geq {\frac {t_{r}} {16}}, U^f_{1}(t)\geq \mu_1-\epsilon \bigg\}, \\
    \cE_{2}(r)&=\big\{\forall i\in [K]\setminus\{1\}, t\le t_r \text{ and } N_{i}(t)>M(r): U_{i}^f(t)< \mu_1-\epsilon, U_{i}^g(t)\leq \mu_1-\epsilon \big\}, \\
    \cE_{3}(r)&=\bigg\{\forall  t \in (t_r/2, t_r], L^g_{1}(t)\geq \mu_1-\epsilon\bigg\}.
\end{align*}

To prove Lemma~\ref{lem:I1-5}, we require the following lemmas,  whose proofs are deferred to the end of this subsection.

\begin{lemma}
\label{lem:I1-1}
For sufficiently small $\delta$, 
   \begin{align*}
       \PP(\cE_{0}(r))\geq 1-\frac{1}{t_{r}^{2}}.
   \end{align*}
\end{lemma}
\begin{lemma}
\label{lem:I1-2}
For sufficiently small $\delta$, 
\begin{align*}
       \PP(\cE_{1}(r))\geq 1-\frac{1}{t_{r}^{2}}.
   \end{align*}
\end{lemma}
\begin{lemma}
\label{lem:I1-3}
For sufficiently small $\delta$, 
\begin{align*}
       \PP(\cE_{2}(r))\geq 1-\frac{1}{t_{r}^{2}}.
   \end{align*}
\end{lemma}
\begin{lemma}
\label{lem:I1-4}
If $\cE_{0}(r)$, $\cE_{1}(r)$ and $\cE_{2}(r)$ are true, the number of pulls of the optimal arm by time $t_r/2$ is at least
\begin{align*}
    \frac{t_r}{16}-KM(r).
\end{align*}
Besides, for sufficiently small $\delta$, 
\begin{align*}
    \PP\left(\cE_{3}(r) \ | \ \cE_{0}(r),\cE_{1}(r),\cE_{2}(r)\right)\geq 1-\frac{1}{t_{r}^2}.
\end{align*}
\end{lemma}

% \begin{proof}[Proof of Lemma \ref{lem:I1-5}]
\begin{proof}\textbf{of Lemma~\ref{lem:I1-5}$\quad$}
Assume  that $\cE_{0}(r)$, $\cE_{1}(r)$, $\cE_{2}(r)$ and $\cE_{3}(r)$ are true. We first divide the time steps in $t\in [t_{r}]$  into three disjoint sets. We let
\begin{align*}
    \cT_{1}(r)=\bigg\{t\in (t_r/2, t_{r}] \ \bigg| \ \max_{i\in [K]\setminus \{1\}} U_{i}^g(t)< \mu_1-\epsilon \ \text{and} \ \max_{i\in [K]\setminus \{1\}} U_{i}^f(t)< U_{1}^{f}(t) \bigg\}
\end{align*}
\begin{align*}
     \cT_{2}^{g}(r)=\left\{t\in (t_r/2, t_r]  \ \bigg | \ \max_{i\in [K]\setminus \{1\}} U_{i}^g(t)\geq \mu_1-\epsilon  \ \text{and} \ \max_{i\in [K]\setminus \{1\}} U_{i}^f(t)< U_{1}^{f}(t) \right\}. 
\end{align*}
and 
\begin{align*}
    \cT_{2}^{f}(r)=\left\{t\in (t_r/2, t_r]  \ \bigg | \  \max_{i\in [K]\setminus \{1\}} U_{i}^f(t)\geq U_{1}^{f}(t)\right\}. 
\end{align*}
%Note that for $t\in \cT_{2}(r)$, then $t\in \cT_{2}^{f}(r)\bigcup \cT_{2}^{g}(r)$. Therefore,
%\begin{align*}
 %   |\cT_{2}(r)|\leq |\cT_{2}^{f}(r)|+|\cT_{2}^{g}(r)|.
%\end{align*}
Let \( Y_{t} \) be the indicator variable such that
\[
Y_{t} =
\begin{cases} 
1 & \text{if the coin toss at time } t \text{ results in heads}, \\ 
0 & \text{if the coin toss at time } t \text{ results in tails}.
\end{cases}
\]
Assume $A_{t}=A_{t}^{g}\in [K]\setminus \{1\}$ occurs for at least $(K-1)M(r)$ times by time $\hat t$. Since $\cE_{2}(r)$ is true, for $t\in (\hat t,t_r]$, 
\begin{align*}
    \max_{i\in [K]\setminus \{1\}} U_{i}^g(t)<\mu_1-\epsilon.  
\end{align*}
Therefore, 
\begin{align*}
    \sum_{t\in \cT_{2}^{g} (r)}\ind\left\{A_{t}^{g}\in [K]\setminus \{1\}, Y_t=0\right\} \leq (K-1)M(r).
\end{align*}
%Since $\cE_{3}(r)$ is true, for $t\in (t_{r}/2, t_{r}]$
%\begin{align*}
%    U_{1}^{g}(t) \geq L_{1}^{g}(t) \geq \mu_1-\epsilon.   
%\end{align*}
Note that 
\begin{align*}
   \EE\bigg[\ind\{A_{t}^g\in [K]\setminus \{1\}, Y_{t}=0\} \ \bigg | \ \ \max_{i\in [K]\setminus \{1\}} U_{i}^f(t)< U_{1}^{f}(t)  \bigg]= \PP(Y_{t}=0)=1-\beta(\delta).
\end{align*}
Consider the independent Bernoulli random variable $\{Z_i\}_{i\geq 1}$ with bias $1-\beta(\delta)$. Then we have that for fixed $L$, 
\begin{align*}
    \PP\bigg(|\cT_{2}^{g}(r)|\geq L \bigg) & \leq \PP\bigg(\sum_{i=1}^{L} Z_{i}<(K-1)M(r) \bigg).
\end{align*}

Let 
\begin{align*}
    E_{g}=\bigg\{\sum_{i=1}^{t_r/8} Z_{i}\geq \frac{t_r}{16}(1-\beta(\delta))\bigg\}.
\end{align*}

Applying Hoeffding's bound,  we have
\begin{align*}
\PP\big(E_{g}^c \big) & \leq \exp \left(-\frac{2\frac{t_r^2}{16^2}(1-\beta(\delta))^2 }{t_r/8} \right) \notag \\
& \lesssim \frac{1}{t_{r}^2}. \tag{for sufficiently small $\delta$}
\end{align*}
where we use $\lesssim$ as shorthand for Big O notation, indicating asymptotic behavior.

Note that for sufficiently small $\delta$, 
\begin{align*}
 \frac{t_r}{16}(1-\beta(\delta)) \geq (K-1) M(r)
\end{align*}
Therefore, setting $L=t_r/8$, we can obtain
\begin{align*}
     \PP\bigg(|\cT_{2}^{g}(r)|< \frac{t_r}{8} \bigg)\geq \PP(\cE_{0}(r), \cE_{1}(r),\cE_{2}(r),\cE_{3}(r),E_{g})\geq 1- \frac{5}{t_{r}^2}. 
\end{align*}
Similarly,  assume $A_{t}=A_{t}^{f}\in [K]\setminus \{1\}$ occurs for at least $(K-1)M(r)$ times by time $\hat t$. Since $\cE_{1}(r)$ and $\cE_{2}(r)$ are true, for $t\in (\hat t,t_r]$, 
\begin{align*}
    \max_{i\in [K]\setminus \{1\}} U_{i}^f(t)<\mu_1-\epsilon\leq U_{1}^{f}(t).  
\end{align*}
Therefore, 
\begin{align*}
    \sum_{t\in \cT_{2}^{f} (r)}\ind\left\{A_{t}^{f}\in [K]\setminus \{1\}, Y_t=1\right\} \leq (K-1)M(r).
\end{align*}
Note that  
\begin{align*}
   \PP\bigg( A_{t}=A_{t}^{f}  \ \bigg | \ \max_{i>1} U_{i}^{f}(t)\geq U_{1}^f \bigg)= \beta(\delta).
\end{align*}
Consider the collection of independent Bernoulli random variables $\{X_i\}_{i\geq 1}$ with bias $\beta(\delta)$. Then we have that for fixed $L$, 
\begin{align*}
    \PP\bigg(|\cT_{2}^{f}(r)|\geq L \bigg) & \leq \PP\bigg(\sum_{i=1}^{L} X_{i}<(K-1)M(r) \bigg).
\end{align*}

Let 
\begin{align*}
    E_{f}=\bigg\{\sum_{i=1}^{t_r/8} X_{i}\geq \frac{t_r}{16}\beta(\delta)\bigg\}.
\end{align*}

Similarly, for sufficiently small $\delta$, we have $\PP(E_{f}^{c})\leq 1/t_{r}^2$ and
\begin{align*}
      \PP\bigg(|\cT_{2}^{f}(r)|< \frac{t_r}{8} \bigg) \geq \PP(\cE_{0}(r),\cE_{1}(r),\cE_{2}(r),\cE_{3}(r),E_{f}) \geq 1-\frac{5}{t_r^2}.
\end{align*}
Now, we have with probability at least $1-{14}/{t_{r}^2}$,  it holds that $$\cE_{0}(r),\cE_{1}(r),\cE_{2}(r),\cE_{3}(r), |\cT_{2}^{f}(r)|< \frac{t_r}{8}, |\cT_{2}^{g}(r)|< \frac{t_r}{8}.$$ 
Therefore, with probability at least $1-14/t_{r}^2$, $$|\cT_{1}(r)|\geq t_r/2-t_r/8-t_r/8>0.$$

Then for any $t\in \cT_{1}(r)$, we have
\begin{enumerate}
    \item $A_{t}^{f}=1$ because $U_{1}^f(t)> \max_{i\in [K]\setminus \{1\}} U_{i}^{f}(t)$;
    \item $L_{1}^{g}(t)\geq \mu_1-\epsilon$ because $\cE_{3}(r)$ is true;
    \item $\max_{i\in [K]\setminus \{1\}} U_{i}^{g}(t)<\mu_1-\epsilon$.
\end{enumerate}

Hence, 
\begin{align*}
    L_{A_{t}^f}^{g}(t) =  L_{1}^{g}(t) \geq \mu_1-\epsilon \geq \max_{i\in [K]\setminus \{1\}} U_{i}^{g}(t),
\end{align*}
which indicates that the algorithm returns. 

This completes the proof of Lemma \ref{lem:I1-5}.
\end{proof}

% \begin{proof}[Proof of Lemma \ref{lem:I1-1}]
\begin{proof}\textbf{of Lemma~\ref{lem:I1-1}$\quad$}
Since $$\beta(\delta) = 1 - \min\left\{ \frac{1}{\gamma}, \frac{1}{2} \right\}\ge \frac 1 2,$$ by Hoeffding's bound,  we have 
\begin{align*}
    \PP\bigg(\cE_{0}(r)^{c} \bigg)  = \PP\bigg(L_{r}\le \frac{t_r}{8}\bigg) &\leq 
    \PP\bigg(L_{r}-\beta(\delta) \cdot \frac {t_{r}} 2 <-\frac{1}{8}t_r\bigg) \\
    &\le \exp\bigg(-\frac{t_r}{32} \bigg) \\
    &\lesssim \frac{1}{t_r^2}
\end{align*}
for sufficiently small $\delta$.

\end{proof}

% \begin{proof} [Proof of Lemma \ref{lem:I1-2}]
\begin{proof}\textbf{of Lemma~\ref{lem:I1-2}$\quad$}
    \begin{align*}
        \PP\bigg(\cE_{1}(r)^{c} \bigg) &=\PP\bigg(\exists t>\frac{t_r}{16}, \hat{\mu}_{1}(t-1)< \mu_1-\epsilon,  \kl(\hat{\mu}_{1}(t-1), \mu_1-\epsilon)\geq \frac{f(t)}{N_{i}(t-1)}\bigg) \notag \\
        &\leq \PP\bigg(\exists s\geq 1, \hat{\mu}_{1s}< \mu_1-\epsilon,  \kl(\hat{\mu}_{1s}, \mu_1-\epsilon)\geq \frac{f(t_r/16)}{s}\bigg) \notag \\
        &\leq   \PP\bigg(\exists s\geq 1, \hat{\mu}_{1s}< \mu_1-\epsilon,  \kl(\hat{\mu}_{1s}, \mu_1)\geq \frac{f(t_r/16)}{s}+{\frac{\epsilon^2}{2V}}\bigg) \tag{due to Lemma \ref{lemma_klmean}} \\
        &\leq  \sum_{s=1}^{\infty} \exp\bigg(-s\bigg({\frac{\epsilon^2}{2V}}+\frac{f(t_r/16)}{s} \bigg) \bigg) \tag{due to Lemma \ref{lem:maximal-inequality}} \\
        &\leq  \frac{1}{e^{f(t_{r}/16)}}\sum_{s=1}^{\infty} \exp\big(-s\epsilon^2/(2V)\big) \notag \\
        &\leq  \frac{2V }{\epsilon^2 \cdot e^{f(t_r/16)}} \notag \\
        &\lesssim  \frac{1}{t_{r}^2}.
    \end{align*}
    \end{proof}

% \begin{proof}[Proof of Lemma \ref{lem:I1-3}]
\begin{proof}\textbf{of Lemma~\ref{lem:I1-3}$\quad$}
From Lemma \ref{lem:maximal-inequality}, we have
\begin{align*}
    \PP(\exists i \in [K]\setminus \{1\},  \exists s> M(r), \hat{\mu}_{is}> \mu_i+\epsilon) 
    &\le \sum_{i \in [K]\setminus \{1\}}{\PP( \exists s> M(r), \hat{\mu}_{is}> \mu_i+\epsilon)}\\
    &\leq K \exp\bigg(-{\frac{M(r)\epsilon^2}{2V}} \bigg) \\
    &\le \frac{1}{t_{r}^2}.
\end{align*}

Then,  with probability at least $1-\frac{1}{t_r^2}$, we have for all $i>1$ and $s>M(r)$,
\begin{align*}
 &\kl(\hat{\mu}_{is}, \mu_1-\epsilon) \geq   \kl(\mu_i+\epsilon, \mu_1-\epsilon) >  \frac{f
 (t_{r})}{s}; \tag{which implies $U^f_{i}(t)<\mu_1-\epsilon$} 
 \\
  &\kl(\hat{\mu}_{is}, \mu_1-\epsilon) \geq   \kl(\mu_i+\epsilon, \mu_1-\epsilon) \geq \frac{g
 (\delta,t_{r})}{s}. \tag{which implies $U_{i}^g(t) \leq \mu_1-\epsilon$}
\end{align*}
\end{proof}

% \begin{proof}[Proof of Lemma \ref{lem:I1-4}]
\begin{proof}\textbf{of Lemma~\ref{lem:I1-4}$\quad$}
Since $\cE_{0}(r)$ and $\cE_{1}(r)$  are true, $L_{r} \geq \frac{t_r}{8}$ and when $t>\frac{t_r}{16}$, $U_{1}(t)>\mu_1-\epsilon$. 
Therefore, among the \(L_{r}\) time steps where the result of the coin toss is heads, there are at least \(L_{r} - \frac{t_{r}}{16} \geq \frac{t_{r}}{16}\) time steps where if \(U_{i}(t) < \mu_1 - \epsilon\) for all \(i > 1\), then \(A_{t} = 1\).

Since $\cE_{2}(r)$ is true, after at most $KM(r)$ pulls on the suboptimal arms,  we have $U_{i}(t)<\mu_1-\epsilon$ for all $i>1$. Therefore, if $\cE_{0}(r), \cE_{1}(r),$ and $\cE_{2}(r)$ are true, the number of pulls of the optimal arm by time $t_{r}/2$ is at least $t_{r}/16-KM({r})$, which completes the proof of the first statement. 

For the second statement, consider any $t\in  (t_r/2, t_r]$. We will show by contradiction that if $\hat{\mu}_{1}(t-1)>\mu_1-\epsilon/2$, then $L^g_{1}(t)\geq \mu_1-\epsilon$.

Suppose that $L^g_{1}(t)< \mu_1-\epsilon$. Then we have 
\begin{align*}
\frac{g(\delta,t_r)}{N_{1}(t-1)} &\ge \frac{g(\delta,t)}{N_{1}(t-1)}\\ 
&\geq  \kl(\hat{\mu}_{1}(t-1),\mu_1-\epsilon)   \notag \\
&\geq \kl(\mu_1-\epsilon/2, \mu_1-\epsilon).
\end{align*}

Therefore, together with Lemma~\ref{lemma_klmean}, we can obtain
\begin{align*}
    N_{1}(t-1)\leq {\frac{g(\delta,t_{r})}{\kl(\mu_1-\epsilon/2,\mu_1-\epsilon)} \leq \frac{8Vg(\delta,t_r)}{\epsilon^2}}.
\end{align*}

However, according to the first statement,
$$
N_{1}(t-1) \ge \frac{t_r}{16}-KM(r).
$$
Note that for sufficiently small $\delta$, 
\begin{align*}
    \frac{t_r}{16}-KM(r) \gtrsim \frac{8Vg(\delta,t_r)}{\epsilon^2},
\end{align*}
which leads to a contradiction. Therefore, we can establish that $L^g_{1}(t)\geq \mu_1-\epsilon$.

Furthermore, from Lemma \ref{lem:maximal-inequality},
\begin{align*}
    \PP(\exists t\in  (t_r/2, t_r], \hat{\mu}_{1}(t-1)<\mu_1-\epsilon/2) \leq {\exp\bigg(\frac{-(t_r/16-KM(r))\cdot \epsilon^2}{8V}\bigg)} \lesssim \frac{1}{t_{r}^2}.
\end{align*}

Altogether, we have for sufficiently small $\delta$, 
\begin{align*}
    \PP\left(\cE_{3}(r) \ |  \ \cE_{0}(r), \cE_{1}(r),\cE_{2}(r)  \right) & \ge \PP(\forall t\in  (t_r/2, t_r], \hat{\mu}_{1}(t-1)\ge \mu_1-\epsilon/2) \\
    &\geq 1-\frac{1}{t_{r}^2}.
\end{align*}
\end{proof}

\section{Proof of the Uniformity in the Allocation in Example~\ref{eg:two_armed}} \label{app:two_armed}

Consider a two-armed Bernoulli bandit instance $\bmu = (1-\mu, \mu)$ with $\mu \in (0, 1/2)$. For any $\delta$-PAC BAI algorithm, it holds that
\begin{equation*}
\liminf_{\delta\to 0} \frac{\E_{\bmu}[ \tau_\delta]} {\log (1/\delta)} \ge \Gamma^{*}(\bmu)
\end{equation*}
where
\begin{align*}
\label{equation_asymptotic_lowerbound_minimalsample_maxmin}
\Gamma^{*}(\bmu)^{-1} :&= \sup _{w \in \mathcal P _{K}} \inf _{\blambda \in \mathrm{Alt}(\bmu)}\left(\sum_{i=1}^K{w_i \mathrm{kl} \left(\mu_{i}, \lambda_{i}\right)}\right) \\
&= \sup _{w \in \mathcal P _{K}} \inf _{\lambda_2 > \lambda_1}\left(w_1  \mathrm{kl} \left(\mu_{1}, \lambda_{1}\right) +  w_2 \mathrm{kl} \left(\mu_{2}, \lambda_{2}\right) \right) \\
&= \sup _{w \in \mathcal P _{K}} \inf _{\mu_1 < \lambda <\mu_2}\left(w_1  \mathrm{kl} \left(\mu_{1}, \lambda\right) +  w_2 \mathrm{kl} \left(\mu_{2}, \lambda\right) \right) \\
&= \sup _{w \in \mathcal P _{K}} \inf _{1-\mu<\lambda <\mu}\left(w_1  \mathrm{kl} \left(1-\mu, \lambda\right) +  w_2 \mathrm{kl} \left(\mu, \lambda\right) \right).
\end{align*}

Substituting the KL divergence for Bernoulli distributions yields
\begin{align*}
\Gamma^{*}(\bmu)^{-1} & = \sup _{w \in \mathcal P _{K}} \inf _{1-\mu<\lambda <\mu}  \Bigg( w_1 \left( (1 - \mu) \log\left(\frac{1 - \mu}{\lambda}\right) + \mu \log\left(\frac{\mu}{1 - \lambda}\right) \right) \\ & \quad \quad \quad \quad  \quad \quad  \quad \quad \quad   + w_2 \left( \mu \log\left(\frac{\mu}{\lambda}\right) + (1 - \mu) \log\left(\frac{1 - \mu}{1 - \lambda}\right) \right) \Bigg)  \\
& = \sup _{w \in \mathcal P _{K}} \inf _{1-\mu<\lambda <\mu} \Big( - \left( w_1(1 - \mu) + w_2 \mu \right) \log \lambda - \left( w_1 \mu + w_2(1 - \mu) \right) \log(1 - \lambda)  \\ 
& \quad \quad \quad \quad  \quad \quad  \quad \quad \quad  +  \left( \mu \log \mu + (1 - \mu) \log(1 - \mu) \right) \Big) 
\end{align*}

Taking the derivative with respect to $\lambda$ reveals that the inner infimum is attained at $\lambda_* := w_1(1 - \mu) + w_2 \mu $. Consequently,
\begin{align*}
\Gamma^{*}(\bmu)^{-1} & = \sup _{w \in \mathcal P _{K}} \Big( - \lambda_* \log \lambda_*- \left( 1-\lambda_*\right) \log(1 - \lambda_*)  +  \left( \mu \log \mu + (1 - \mu) \log(1 - \mu) \right) \Big) .
\end{align*}

It is readily seen that the outer supremum is maximized when $\lambda_* = 1-\lambda_*$, which implies $w_1=w_2=1/2$. Therefore, the optimal allocation is uniform for \textbf{all} $\mu \in (0, 1/2)$. (The special case that $\mu =1/2$ is excluded from our analysis as the optimal arm is no longer unique.)

%%%%%%%%%%%%%%%%%%%%%%%%%%%%%%%%%%%%%%%%%%%%%%%%%%%%%%%%%%%%

\vskip 0.2in
\bibliography{references}

\end{document}